\documentclass{article} 
\usepackage{iclr2027_conference,times}

\usepackage{amsmath,amsfonts,bm}

\def\eqref#1{equation~\ref{#1}}

\def\1{\bm{1}}

\def\vs{{\bm{s}}}

\def\vx{{\bm{x}}}
\def\vy{{\bm{y}}}

\DeclareMathAlphabet{\mathsfit}{\encodingdefault}{\sfdefault}{m}{sl}
\SetMathAlphabet{\mathsfit}{bold}{\encodingdefault}{\sfdefault}{bx}{n}

\usepackage{hyperref}
\usepackage{url}
\usepackage{graphicx}
\usepackage{enumitem}
\usepackage{booktabs}
\usepackage{tabularx}
\usepackage{array}
\usepackage{xcolor}
\usepackage{colortbl}
\usepackage{threeparttable}
\usepackage{subcaption}
\usepackage{amssymb}
\usepackage{amsthm}
\usepackage{amsmath}
\newtheorem{proposition}{Proposition}
\newtheorem*{proposition*}{Proposition}

\usepackage{wrapfig}

\usepackage{adjustbox}
\usepackage{makecell}
\usepackage{multirow}

\definecolor{linkgray}{HTML}{777777}

\title{Not Every Token Is Worth Distilling:\\ Selective Supervision for Direct-OPD}

\newcommand{\githubicon}{%
  \raisebox{-0.15em}{%
    \includegraphics[height=1.05em]{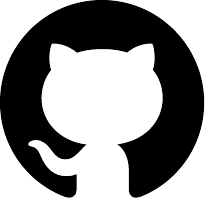}%
  }%
}

\newcommand{\hficon}{%
  \raisebox{-0.18em}{%
    \includegraphics[height=1.08em]{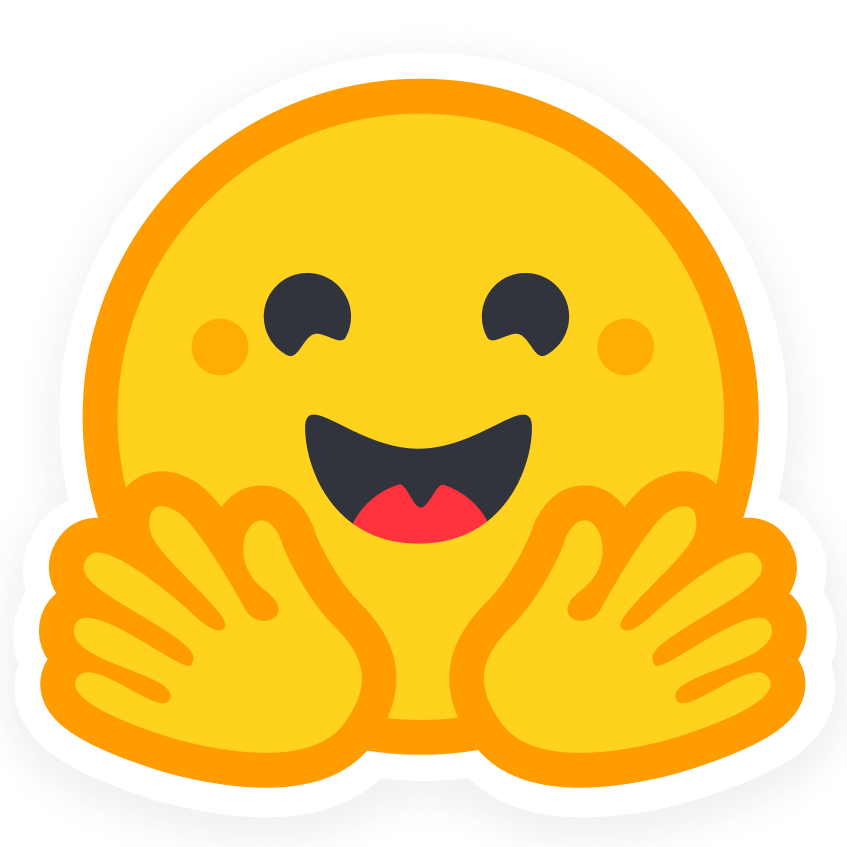}%
  }%
}

\iclrfinalcopy 
\begin{document}

\author{Yibo Zhao\thanks{Equal contribution. The first two authors may list their names
in either order on their CVs.},\quad Zixuan Yang\footnotemark[1],\quad Yunshi Lan,\quad Xiang Li\thanks{Corresponding author: \texttt{xiangli@dase.ecnu.edu.cn}.} \\
School of Data Science and Engineering\\
East China Normal University\\[5pt]
{\small\normalfont
\noindent
\hspace*{-0.8em}%
\makebox[\dimexpr\textwidth-2\tabcolsep\relax][c]{%
  \href{https://huggingface.co/SelectiveDOPD}{%
    \hficon\hspace{0.20em}\textcolor{linkgray}{Hugging Face}%
  }%
  \hspace{2.2em}%
  \href{https://github.com/Luli3220/S2D-OPD}{%
    \githubicon\hspace{0.20em}\textcolor{linkgray}{Github}%
  }%
}}
}

\maketitle


\begin{abstract}
Direct On-Policy Distillation (Direct-OPD) transfers reinforcement-learning-induced policy improvements from a small model to a larger student by using the token-level log-ratio between post-RL and pre-RL checkpoints as dense supervision on the student's own rollouts.
This transfer rewards the policy shift at every state, yet the log-ratio measures only relative change: it can stay fixed even as the probability mass that both checkpoints assign to the student's candidate tokens vanishes.
Through an exact construction, we show that the Direct-OPD reward and its update can remain unchanged while the Jensen--Shannon divergence (JSD) and both KL directions between the checkpoints vanish with this mass, and we note that a small JSD bounds how much the teacher's behavior changed.
Motivated by this analysis, we propose Selective Supervision for Direct-OPD (S$^2$D-OPD), which ranks student-sampled states by their teacher--reference JSD and masks Direct-OPD supervision at low-divergence states, retaining only the top 10\% of states per response. Across two teacher pairs and four student models ranging from 1.7B to 8B parameters, S$^2$D-OPD improves held-out accuracy over dense Direct-OPD on AIME and HMMT benchmarks in seven of eight settings and matches it in the eighth, without extra forward passes. 
\end{abstract}

\begin{quote}
\textit{``No Free Lunch for Supervised Machine Learning.''}\hfill --- David H. Wolpert
\end{quote}

\section{Introduction}

Guided by scaling laws~\citep{scalinglaw1,scalinglaw2,scalinglaw3}, recent large language models have continued to scale up pre-training~\citep{deepseek,kimi,glm5}, which broadens their knowledge and latent capabilities. 
Post-training, most prominently reinforcement learning (RL)~\citep{rl1,grpo,rl3}, is then needed to elicit and refine these capabilities.
As model size grows, RL demands more rollout generation, training compute, memory, and infrastructure~\citep{llamarl}, and stable optimization remains difficult to achieve~\citep{instable}.
Consequently, the models with the greatest post-training potential are also the most expensive to improve through RL.

One line of work makes large-scale RL more stable through improved training algorithms~\citep{dapo,gspo,cispo,r3,sao} and more efficient through training and inference systems~\citep{verl,areal,vllm,megatron,slime}.
These advances make RL on large models more practical, but its cost still grows with model size.
Another line of work uses on-policy distillation (OPD)~\citep{GKD,minillm,opd}, in which a stronger teacher provides token-level supervision on states sampled by the student.
OPD transfers capabilities efficiently to a smaller student~\citep{rethinkopd,rethinkopd2}, but it relies on a teacher stronger than the student, which is unavailable when the target is already the strongest model.
Together, these approaches leave a challenge open: how to improve a large model without paying the cost of RL at its scale.


Direct-OPD~\citep{direct-opd} and Proxy-OPD~\citep{ProxyOPD} recently proposed a weak-to-strong route around this challenge: running RL on a small model and transferring the result to a larger one.
Both methods extract the RL-induced policy shift as the token-level log-ratio between the post-RL teacher and its pre-RL reference, and use it as an OPD reward on states sampled by the larger student.
This turns outcome-level rewards into dense token-level supervision for the large model, without running RL at its scale or requiring a teacher stronger than the student.

However, this apparent free lunch leaves unexamined \emph{whether every token-level reward reflects a meaningful change in the teacher's behavior}.
Because the log-ratio reward measures only relative change, it can stay fixed even as the probability mass that both checkpoints assign to the student's top candidates vanishes.
Direct-OPD can therefore reward a state where both checkpoints barely support these candidates as strongly as one where RL clearly changed the teacher's behavior.
This raises the question that motivates this work: \emph{is every state's policy shift worth distilling for free?}

We make this probability-mass mismatch exact in Sec.~\ref{sec:MethodMotivation}: as the mass that both checkpoints assign to the student's top candidates vanishes, the Direct-OPD reward and gradient can stay fixed, whereas the Jensen--Shannon divergence (JSD) and both directions of KL divergence between the checkpoints vanish with it.
Because JSD accounts for the probability mass that the log-ratio ignores, we propose \textbf{\underline{S}elective \underline{S}upervision for \underline{D}irect-OPD} (S$^2$D-OPD), which ranks student-sampled states by teacher--reference JSD and masks Direct-OPD supervision at low-divergence states.
Empirically, the answer to our question is no: retaining only the top 10\% of states per response, S$^2$D-OPD improves held-out accuracy over dense Direct-OPD in seven of eight teacher--student settings and matches it in the eighth (mean gain 0.95 points; 95\% CI 0.40--1.54), without extra forward passes.

In summary, our contributions are threefold:

\begin{itemize}[leftmargin=*]
\item \textbf{A Probability-Mass Mismatch in Direct-OPD.} Through an exact construction, we show that the log-ratio reward and its local gradient on the student can remain fixed while the probability mass behind the policy shift vanishes, and with it the teacher--reference JSD and both KL directions. We further show that JSD bounds how much the teacher's behavior can change at every state, which motivates selecting states by divergence rather than by the reward itself.

\item \textbf{Stable and Effective Selective Transfer.} We propose S$^2$D-OPD, which masks Direct-OPD supervision at low-divergence states during policy transfer. Across four student scales and two teacher pairs, it improves held-out accuracy over Direct-OPD in seven of eight settings and yields smoother late-stage validation curves under the JustRL teacher pair.


\item \textbf{Understanding the Gains from Selective Transfer.} Within a fixed teacher--student setting, performance broadly rises with JSD percentile: the top bin outperforms a uniformly sampled 10\% subset, whereas the lowest bin degrades the student below its initialization and eventually collapses. Across teacher pairs, the pair with lower overall JSD benefits more from masking, consistent with low-divergence filtering being a source of the improvement over dense Direct-OPD.
\end{itemize}

\section{Related Work}

\textbf{On-Policy Distillation.}
OPD trains a student on prefixes sampled from its own policy, using the teacher's next-token distributions as dense supervision at every position~\citep{GKD,minillm,opd}. 
Subsequent work refines it through alternative objectives~\citep{TeacherSide1,Aopd}, stabilization strategies~\citep{rethinkopd,revisitopd}, and privileged-context self-distillation~\citep{opsd,RLCSD}.
Despite their differences, these methods primarily learn from the teacher's policy itself, limiting transfer to improvements already present in the teacher; recent work therefore targets the teacher's
policy shift instead.
ExOPD~\citep{Exopd} extrapolates the teacher's improvement over its reference model to construct a target beyond the teacher.
Direct-OPD~\citep{direct-opd} and Proxy-OPD~\citep{ProxyOPD} instead transfer the log-ratio between a reward-optimized checkpoint and its pre-RL reference, analogous to the logit shifts induced by fine-tuning studied in CMC~\citep{CMC}.
This targets the RL-induced policy shift and can provide useful supervision even for students already stronger than the post-RL teacher.
However, token-level log-ratios capture relative changes but are insensitive to the absolute probability mass supporting these changes.
We examine this probability-mass mismatch in Direct-OPD and use teacher--reference divergence to select supervision positions.

\textbf{Token Selection in Policy Distillation.}
Selective distillation asks which positions are worth training on, and existing criteria differ mainly in which distributions they read.
Some read the student alone, prioritizing positions where it is uncertain~\citep{SE-KD,REOPOLD}; others the teacher alone, weighting by its confidence or local margin~\citep{TeacherSide1,TeacherSide2}; a third group compares the two, emphasizing teacher–student disagreement, which TIP~\citep{tip} organizes through an entropy--divergence taxonomy and TA-OPD~\citep{Ta-opd} restricts to the student's predictive support.
We consider the policy change from a pre-RL reference to a post-RL
teacher at student-visited prefixes.
A related approach, OPD$^2$~\citep{opdd}, gates sampled-token delta updates by sign agreement between the centered teacher--base and teacher--student log-ratios.
Our selection criterion instead ranks positions by teacher--reference JSD, accounting for the probability mass underlying the policy shift while retaining the Direct-OPD update at selected positions.

\section{Preliminaries}\label{sec:preliminary}

\textbf{Setting.} We consider three policies: a pre-RL reference $\pi_{\mathrm{ref}}$, a post-RL teacher $\pi_{\mathrm{T}}$ obtained from $\pi_{\mathrm{ref}}$ by outcome-based RL such as GRPO~\citep{grpo}, and a larger student $\pi_\theta$ initialized at $\pi_{\mathrm{stu}}$.
In our experiments, both teacher-side checkpoints are publicly released (Sec.~\ref{sec:experiment_setting}), so we run no RL ourselves.
Given a prompt $\vx\sim\mathcal D$ and a response $\vy=(y_1,\ldots,y_{|\vy|})$ sampled from the student, position $t$ has state $\vs_t=(\vx,\vy_{<t})$.

\textbf{Direct-OPD} treats the teacher's RL-induced policy shift as a dense reward for the student.
For any token $v$, the reward at state $\vs_t$ is the teacher--reference log-ratio
\begin{equation}
    \Delta_t(v\mid\vs_t)
    = \log \frac{\pi_{\mathrm{T}}(v \mid \vs_t)}
                 {\pi_{\mathrm{ref}}(v \mid \vs_t)},
    \label{eq:policy_shift}
\end{equation}
which is positive where RL increased the probability of $v$ and negative where it decreased it. 

Direct-OPD maximizes this reward on states visited by the student, with KL regularization:
\begin{equation}
    J_{\text{Direct-OPD}}(\theta)
    = \mathbb{E}_{\vx\sim\mathcal D,\,(\vs_t,y_t)\sim\pi_\theta}
    \left[
        \Delta_t(y_t\mid\vs_t)
        - \alpha D_{\mathrm{KL}}\left[
            \pi_\theta(\cdot\mid\vs_t)
            \Vert
            \pi_{\mathrm{stu}}(\cdot\mid\vs_t)
        \right]
    \right].
    \label{eq:direct_opd_objective}
\end{equation}
Here, $\alpha>0$ controls KL regularization, and the expectation
covers all valid response positions. 

\textbf{Top-$K$ implementation.} In practice, Direct-OPD evaluates the reward on the student's top-$K$ candidates at each state rather than only on the sampled token:
\begin{equation}
    \mathcal V_K(\vs_t) = \operatorname{TopK}(\pi_\theta(\cdot\mid \vs_t),K),
    \quad
    \bar p_t(v)=\frac{\pi_\theta(v\mid\vs_t)}{\sum_{u\in\mathcal V_K(\vs_t)}\pi_\theta(u\mid\vs_t)},
\end{equation}
where $\operatorname{TopK}(\cdot, K)$ returns the set of $K$ tokens with the largest probabilities. Each candidate receives the teacher--reference reward $\Delta_t(v\mid \vs_t)$, weighted by its renormalized student probability $\bar p_t(v)$. 
With the state and candidate set held fixed, the implemented local reward-gradient contribution is:
\begin{equation}
    g_t^R=\sum_{v\in\mathcal V_K(\vs_t)}\operatorname{sg}\left[
    \bar p_t(v)\log\frac{\pi_\mathrm T(v\mid\vs_t)}{\pi_\mathrm{ref}(v\mid\vs_t)}
    \right]
    \nabla_\theta \log\pi_\theta(v\mid\vs_t),
    \label{eq:reward_grad}
\end{equation}
where $\operatorname{sg}$ denotes stop-gradient.
The student-anchor KL term supplies a separate regularization gradient.
Appendix~\ref{app:selective_DOPD_Objective} describes its implementation and adaptive coefficient.

\section{Method}
\label{sec:method}




Direct-OPD applies its log-ratio reward at every valid position of a student response.
Sec.~\ref{sec:MethodMotivation} shows that this reward can ignore the probability mass behind the teacher's policy shift, and Sec.~\ref{sec:methods} introduces S$^2$D-OPD, which selects states by teacher--reference divergence.

\subsection{Theoretical Motivation: Probability-Mass Mismatch}
\label{sec:MethodMotivation}

At each state, Direct-OPD weights the teacher--reference rewards $\Delta_t(v\mid\vs_t)$ by the student's renormalized probabilities $\bar p_t(v)$ over its top-$K$ candidates in  Eq.~\ref{eq:reward_grad}.
These weights reflect the student's preferences, whereas the rewards encode relative changes in the teacher's policy.
Neither depends on how much probability mass the teacher and reference place on these candidates: rescaling both by a common factor leaves every log-ratio unchanged.
The following exact construction makes this precise: the teacher--reference divergence can vanish while the Direct-OPD update stays fixed.

\paragraph{An exact construction.}
Fix a state $\vs$ and a student checkpoint $\theta_0$ with top-$K$ candidate set $\mathcal V_K=\mathcal V_K(\vs)$.
Let $\mathbf t\neq\mathbf q$ be strictly positive probability vectors on $\mathcal V_K$, let $\mathbf b$ be a strictly positive probability vector on its complement, and for $0<\epsilon<1$ define the teacher and reference:
\begin{equation}
P_\epsilon(v)=
\begin{cases}
\epsilon t_v, & v\in\mathcal V_K,\\
(1-\epsilon)b_v, & v\notin\mathcal V_K,
\end{cases}\qquad
Q_\epsilon(v)=
\begin{cases}
\epsilon q_v, & v\in\mathcal V_K,\\
(1-\epsilon)b_v, & v\notin\mathcal V_K.
\end{cases}
\label{eq:mass_pair}
\end{equation}
Here, $\epsilon$ is the probability mass that each checkpoint assigns to the student's candidates, while the student, candidate set, and conditional distributions stay fixed.
Then, for every candidate $v\in\mathcal V_K$,
\begin{equation}
\Delta_\epsilon(v\mid\vs)=\log\frac{t_v}{q_v},
\qquad
D_{\mathrm{JS}}(P_\epsilon,Q_\epsilon)=\epsilon\, D_{\mathrm{JS}}(\mathbf t,\mathbf q),
\label{eq:mass_invariance}
\end{equation}
and both directions of KL scale with $\epsilon$ in the same way.
As $\epsilon\to0$, all three divergences vanish, whereas every candidate reward, and hence the Direct-OPD update of Eq.~\ref{eq:reward_grad}, stays fixed and nonzero.
The construction thus isolates a single degree of freedom, the probability mass behind the policy shift, to which the Direct-OPD update is insensitive but the divergences are not.
We state this result formally in Prop.~\ref{prop:mass} and prove it in App.~\ref{app:mass_proof}, including why the update is nonzero.

\paragraph{Divergence bounds the behavioral change.}
The construction shows that the Direct-OPD reward can ignore divergence; conversely, divergence bounds how much the teacher's behavior can change.
For any two distributions $P$ and $Q$ on a finite set and any event $A$,
\begin{equation}
|P(A)-Q(A)|\;\leq\;D_{\mathrm{TV}}(P,Q)\;\leq\;\sqrt{2\,D_{\mathrm{JS}}(P,Q)},
\label{eq:tv_bound}
\end{equation}
where $D_{\mathrm{TV}}(P,Q)=\max_{A}|P(A)-Q(A)|$ is the total variation distance and $D_{\mathrm{JS}}$ is measured in nats.
The second inequality follows from Pinsker's inequality applied to each term of $D_{\mathrm{JS}}(P,Q)=\frac12 D_{\mathrm{KL}}(P\Vert M)+\frac12 D_{\mathrm{KL}}(Q\Vert M)$ with $M=(P+Q)/2$, since $D_{\mathrm{TV}}(P,M)=D_{\mathrm{TV}}(Q,M)=\frac12 D_{\mathrm{TV}}(P,Q)$.
Unlike the construction, this bound holds at every state: wherever the teacher--reference JSD is small, the teacher assigns nearly the same probability as the reference to every token and every set of tokens.
Together, the two results characterize what low-divergence masking removes: states at which the teacher's behavior provably changed little, yet at which the Direct-OPD update can be as large as anywhere else.
They do not show that removing these states improves transfer, which Sec.~\ref{sec:section_ladder} tests by training on JSD percentile bins.

\subsection{Divergence-Guided State Selection}
\label{sec:methods}
Sec.~\ref{sec:MethodMotivation} shows that the Direct-OPD reward is insensitive to the probability mass behind a policy shift, whereas the teacher--reference JSD bounds how much the teacher's behavior changed at a state.
S$^2$D-OPD therefore scores each student-sampled state by this divergence and retains Direct-OPD supervision only at the highest-scoring states within each response (Fig.~\ref{fig:method}).
We use JSD as the default score because, unlike KL, it is symmetric, bounded by $\log 2$, and finite when either checkpoint assigns a token zero probability; the procedure applies unchanged to other divergences.

\begin{figure}[tbp]
\centering
\includegraphics[width=\linewidth]{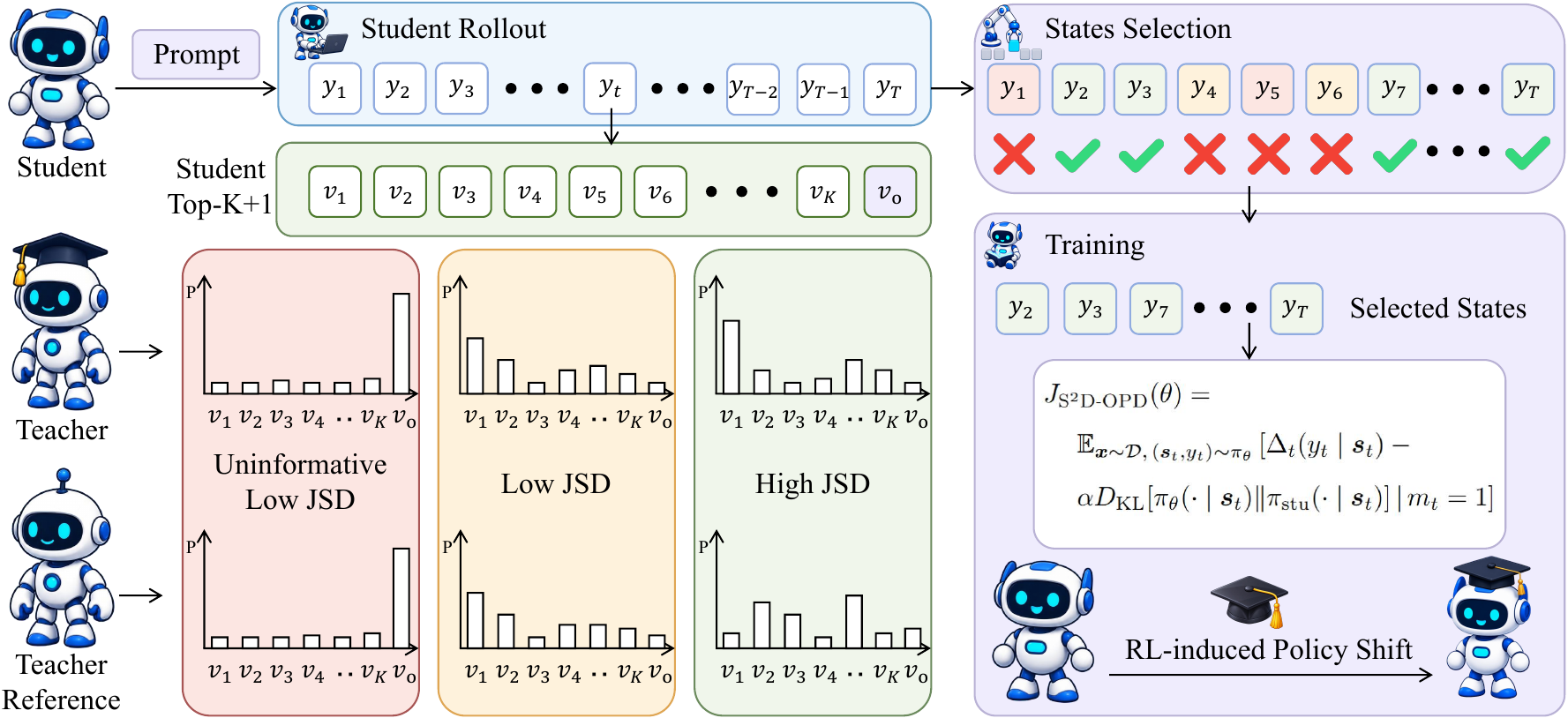}
    \caption{Overview of S$^2$D-OPD. At each student-sampled state, teacher--reference JSD is computed over the student's top-$K$ candidates, with the remaining probability mass grouped into a residual token $v_{\mathrm o}$. The highest-JSD states within each response are retained for the Direct-OPD objective.}
\label{fig:method}
\end{figure}

\paragraph{Scoring on the student's candidates.}
Rather than the full vocabulary, we evaluate the divergence on the top-$K$ candidate set $\mathcal V_K(\vs_t)$, where Direct-OPD already evaluates both checkpoints, and add a residual token $v_{\mathrm o}$ that collects all remaining probability mass.
For $M\in\{\mathrm T,\mathrm{ref}\}$, define
\begin{equation}
\label{eq:topk_dist}
\widetilde{\pi}_M(v\mid\vs_t)=
\begin{cases}
\pi_M(v\mid\vs_t),
& v\in\mathcal V_K(\vs_t),\\[3pt]
1-\displaystyle\sum_{u\in\mathcal V_K(\vs_t)}
\pi_M(u\mid\vs_t),
& v=v_{\mathrm o}.
\end{cases}
\end{equation}
This keeps each candidate's probability and the total residual mass without renormalization.
By the data-processing inequality, the coarsening cannot increase JSD:
\begin{equation}
\label{eq:jsd_score}
\tilde{d}_t\triangleq
D_{\mathrm{JS}}\!\left(
\widetilde{\pi}_{\mathrm T}(\cdot\mid \vs_{t})
\,\middle\|\,
\widetilde{\pi}_{{\mathrm{ref}}}(\cdot\mid \vs_{t})
\right)
\leq
D_{\mathrm{JS}}\!\left(
{\pi}_\mathrm T(\cdot\mid \vs_{t})
\,\middle\|\,
{\pi}_{{\mathrm{ref}}}(\cdot\mid \vs_{t})
\right).
\end{equation}
Fig.~\ref{fig:method} illustrates two ways in which the score can be small.
Either both checkpoints place little mass on the student's candidates, so that both compressed distributions concentrate on $v_{\mathrm o}$, which is the regime of the construction in Sec.~\ref{sec:MethodMotivation}; or both place substantial mass on the candidates and agree on how it is distributed.
The score treats these cases alike, and it need not separate them: because Eq.~\ref{eq:tv_bound} holds for any pair of distributions, including the compressed ones, a small score implies in either case that the teacher barely changed its behavior on the student's candidates.
Conversely, the score is large only when the checkpoints disagree on the individual candidate probabilities or on the total mass assigned to the candidate set.
The score thus measures behavioral change at the resolution of the student's candidates: differences among tail tokens outside the candidate set do not affect it.

\paragraph{Per-response selection.}
We select states within each response, so that rollouts with different overall divergence levels all contribute supervision.
For valid response positions $\mathcal T$ and a common retention ratio $\rho\in(0,1]$, we retain the $k$ states with the largest scores:
\begin{equation}
k
=
\max\left\{
1,\,
\left\lceil\rho|\mathcal T|\right\rceil
\right\},
\qquad
\mathcal I
=
\operatorname{TopK}_{t\in\mathcal T}
(\tilde d_{t},k),
\label{eq:selection}
\end{equation}
where retaining at least one state ensures that every rollout is represented in the objective.
Selecting within each response also makes each mask independent of the rest of the batch.

\paragraph{Masked objective.}
We apply the mask $m_t=\mathbf 1\{t\in\mathcal I\}$ to both terms of the Direct-OPD objective:
\begin{equation}
\label{eq:sparse_objective}
J_{\mathrm{S^2D\text{-}OPD}}(\theta)
=
\mathbb E_{\vx\sim\mathcal D,\,
(\vs_t,y_t)\sim\pi_\theta}
\left[
\Delta_t(y_t\mid\vs_t)
-\alpha D_{\mathrm{KL}}\!\left[
\pi_\theta(\cdot\mid\vs_t)
\Vert
\pi_{\mathrm{stu}}(\cdot\mid\vs_t)
\right]
\,\middle|\,
m_t=1
\right].
\end{equation}
In practice, we hold the selection fixed during optimization and average over all retained states in the batch, which gives the update
\begin{equation}
\label{eq:masked_update}
\mathbf g^{\mathrm{S^2D\text{-}OPD}}
=
\frac{1}{|\mathcal I_{\mathcal B}|}
\sum_{t\in\mathcal I_{\mathcal B}}
\Big[
g_t^R
-\alpha\,\nabla_\theta D_{\mathrm{KL}}\!\left[
\pi_\theta(\cdot\mid\vs_t)
\Vert
\pi_{\mathrm{stu}}(\cdot\mid\vs_t)
\right]
\Big],
\end{equation}
where $\mathcal I_{\mathcal B}$ collects the retained states of all responses in the batch and $g_t^R$ is the reward gradient of Eq.~\ref{eq:reward_grad}.
Masked states thus receive neither the reward nor the student anchor, and $\rho=1$ recovers Direct-OPD.
The adaptive KL coefficient is still computed from all valid positions; App.~\ref{app:selective_DOPD_Objective} gives the remaining optimization details.
Because the score reuses the teacher and reference probabilities that Direct-OPD already computes, S$^2$D-OPD requires no extra forward passes.

\begingroup


\definecolor{QwenAccent}{HTML}{1F4E79}
\definecolor{QwenGainGreen}{HTML}{2F6B4F}
\definecolor{QwenMutedGray}{HTML}{777777}

\definecolor{QwenSelectionBg}{HTML}{F6F6F6}
\definecolor{QwenTestBg}{HTML}{EEF4F8}


\newcommand{\selResult}[1]{%
    \textcolor{QwenMutedGray}{#1}%
}

\newcommand{\selBestResult}[1]{%
    \textbf{\textcolor{QwenMutedGray}{#1}}%
}

\newcommand{\testResult}[1]{%
    #1%
}

\newcommand{\testBestResult}[1]{%
    \textbf{#1}%
}

\newcommand{\testGain}[1]{%
    {\tiny\textcolor{QwenGainGreen}{\enspace(#1)}}%
}

\newcommand{\testAvg}[2]{%
    #1\testGain{#2}%
}

\newcommand{\testBestAvg}[2]{%
    \textbf{#1}\testGain{#2}%
}

\newcommand{\oursmethod}{%
    \textcolor{QwenAccent}{\textbf{+ S$^2$D-OPD}}%
}

\newcommand{\directmethod}{%
    + Direct-OPD%
}

\begin{table}[t]
    \centering
    \scriptsize
    \setlength{\tabcolsep}{2.5pt}
    \renewcommand{\arraystretch}{1.08}

\caption{
    Main results (Avg@32) across four student models and two
    teacher--reference pairs, comparing dense Direct-OPD with
    S$^2$D-OPD using top-10\% JSD selection.
    Checkpoints are selected on AIME24/25 and evaluated on held-out
    AIME26 and HMMT (Nov.\ 2025 and Feb.\ 2026).
    \textbf{Test Avg.} is the average accuracy over 93 held-out
    problems; parentheses denote gains over the initial student.
}
    \label{tab:main_result}

    \begin{adjustbox}{width=\linewidth}
    \begin{tabular}{@{}lcccccccccc@{}}

        \toprule

        &
        \multicolumn{5}{c}{
            \textcolor{QwenAccent}{
                \textbf{R1-Distill-1.5B $\rightarrow$ JustRL-1.5B}
            }
        }
        &
        \multicolumn{5}{c}{
            \textcolor{QwenAccent}{
                \textbf{Nemotron-1.5B $\rightarrow$ QuestA-1.5B}
            }
        }
        \\

        \cmidrule(lr){2-6}
        \cmidrule(lr){7-11}

        &
        \multicolumn{2}{c}{
            \cellcolor{QwenSelectionBg}
            \textbf{Checkpoint Selection}
        }
        &
        \multicolumn{3}{c}{
            \cellcolor{QwenTestBg}
            \textbf{Held-out Evaluation}
        }
        &
        \multicolumn{2}{c}{
            \cellcolor{QwenSelectionBg}
            \textbf{Checkpoint Selection}
        }
        &
        \multicolumn{3}{c}{
            \cellcolor{QwenTestBg}
            \textbf{Held-out Evaluation}
        }
        \\

        \cmidrule(lr){2-3}
        \cmidrule(lr){4-6}
        \cmidrule(lr){7-8}
        \cmidrule(lr){9-11}

        \textbf{Student / Method}
        &
        \cellcolor{QwenSelectionBg}\textbf{AIME24}
        &
        \cellcolor{QwenSelectionBg}\textbf{AIME25}
        &
        \cellcolor{QwenTestBg}\textbf{AIME26}
        &
        \cellcolor{QwenTestBg}\textbf{HMMT}
        &
        \cellcolor{QwenTestBg}\textbf{Test Avg.}
        &
        \cellcolor{QwenSelectionBg}\textbf{AIME24}
        &
        \cellcolor{QwenSelectionBg}\textbf{AIME25}
        &
        \cellcolor{QwenTestBg}\textbf{AIME26}
        &
        \cellcolor{QwenTestBg}\textbf{HMMT}
        &
        \cellcolor{QwenTestBg}\textbf{Test Avg.}
        \\

        \midrule


        \textbf{Qwen3-1.7B}
        & \selResult{49.3}
        & \selResult{36.7}
        & 37.7
        & 28.2
        & 31.3
        & \selResult{49.3}
        & \selResult{36.7}
        & 37.7
        & 28.2
        & 31.3
        \\

        \hspace{0.8em}\directmethod
        & \selResult{59.7}
        & \selResult{42.9}
        & 45.4
        & 32.0
        & \testAvg{36.4}{+5.1}
        & \selResult{58.6}
        & \selResult{43.4}
        & 46.8
        & 31.6
        & \testAvg{36.5}{+5.2}
        \\

        \hspace{0.8em}\oursmethod
        & \selBestResult{61.4}
        & \selBestResult{44.5}
        & \testBestResult{47.4}
        & \testBestResult{32.9}
        & \testBestAvg{37.6}{+6.3}
        & \selBestResult{59.5}
        & \selBestResult{43.8}
        & \testBestResult{48.1}
        & \testBestResult{33.3}
        & \testBestAvg{38.1}{+6.8}
        \\

        \midrule


        \textbf{Qwen3-4B}
        & \selResult{73.2}
        & \selResult{65.2}
        & 64.8
        & 45.9
        & 52.0
        & \selResult{73.2}
        & \selResult{65.2}
        & 64.8
        & 45.9
        & 52.0
        \\

        \hspace{0.8em}\directmethod
        & \selBestResult{78.1}
        & \selResult{70.3}
        & 66.8
        & \testBestResult{47.1}
        & \testAvg{53.5}{+1.5}
        & \selResult{77.6}
        & \selBestResult{70.4}
        & 70.4
        & 47.1
        & \testAvg{54.6}{+2.6}
        \\

        \quad \oursmethod
        & \selResult{77.0}
        & \selBestResult{70.9}
        & \testBestResult{68.7}
        & 46.5
        & \testBestAvg{53.7}{+1.7}
        & \selBestResult{78.0}
        & \selResult{68.3}
        & \testBestResult{70.7}
        & \testBestResult{49.3}
        & \testBestAvg{56.2}{+4.2}
        \\

        \midrule


        \textbf{Qwen3-8B}
        & \selResult{77.3}
        & \selResult{66.0}
        & 67.5
        & 49.7
        & 55.4
        & \selResult{77.3}
        & \selResult{66.0}
        & 67.5
        & 49.7
        & 55.4
        \\

        \hspace{0.8em}\directmethod
        & \selResult{77.5}
        & \selResult{72.2}
        & 69.8
        & \testBestResult{50.2}
        & \testAvg{56.5}{+1.1}
        & \selResult{75.9}
        & \selBestResult{71.7}
        & 69.7
        & 48.6
        & \testAvg{55.4}{+0.0}
        \\

        \hspace{0.8em}\oursmethod
        & \selBestResult{78.0}
        & \selBestResult{73.4}
        & \testBestResult{70.7}
        & 50.0
        & \testBestAvg{56.7}{+1.3}
        & \selBestResult{78.1}
        & \selBestResult{71.7}
        & \testBestResult{70.9}
        & \testBestResult{50.6}
        & \testBestAvg{57.1}{+1.7}
        \\

        \midrule


        \textbf{R1-Distill-7B}
        & \selResult{56.7}
        & \selResult{40.5}
        & 48.2
        & 29.5
        & 35.6
        & \selResult{56.7}
        & \selResult{40.5}
        & 48.2
        & 29.5
        & 35.6
        \\

        \hspace{0.8em}\directmethod
        & \selResult{63.6}
        & \selResult{45.7}
        & \testBestResult{56.8}
        & 32.8
        & \testBestAvg{40.5}{+4.9}
        & \selBestResult{60.7}
        & \selBestResult{42.3}
        & 49.5
        & 31.9
        & \testAvg{37.6}{+2.0}
        \\

        \hspace{0.8em}\oursmethod
        & \selBestResult{64.6}
        & \selBestResult{47.1}
        & 55.0
        & \testBestResult{33.6}
        & \testBestAvg{40.5}{+4.9}
        & \selResult{60.6}
        & \selBestResult{42.3}
        & \testBestResult{52.4}
        & \testBestResult{32.0}
        & \testBestAvg{38.6}{+3.0}
        \\

        \bottomrule

    \end{tabular}
    \end{adjustbox}

\end{table}

\endgroup

\section{Experiments}
\label{sec:experiment}
Sec.~\ref{sec:MethodMotivation} shows which states low-divergence masking removes, but not whether removing them improves transfer.
We test this through three research questions (further analyses in App.~\ref{app:token_overlap},~\ref{app:threshold_sensitivity},~\ref{app:case_study}):
\begin{itemize}
\item[\textbf{RQ1:}] Does masking low-JSD states improve transfer? (Sec.~\ref{sec:main_results})
\item[\textbf{RQ2:}] How does JSD relate to supervision utility and selective masking gains? (Sec.~\ref{sec:section_ladder})
\item[\textbf{RQ3:}] Are the gains robust across divergence measures and selection schemes? (Sec.~\ref{sec:extra_exp})
\end{itemize}

\subsection{Experimental Setup}
\label{sec:experiment_setting}
We evaluate S$^2$D-OPD with two teacher pairs, R1-Distill-1.5B $\rightarrow$ JustRL-1.5B~\citep{justRL} and Nemotron-1.5B $\rightarrow$ QuestA-1.5B~\citep{questa}, and transfer each policy shift to four students: Qwen3-1.7B, Qwen3-4B, Qwen3-8B\footnote{{https://huggingface.co/Qwen/Qwen3-\{1.7,4,8\}B}}, and R1-Distill-7B\footnote{https://huggingface.co/deepseek-ai/DeepSeek-R1-Distill-Qwen-7B}.
All students are trained on Skywork-OR1-RL-Data~\citep{skywork_dataset}.
Following the observation of~\citet{sparsebutcritical} that RL-induced policy shifts are sparse, with large divergence concentrated in a small fraction of tokens, we retain the top 10\% of states per response ($\rho=0.1$) by default.
We select checkpoints on AIME 2024 and AIME 2025 and evaluate the selected checkpoint on three held-out benchmarks: AIME 2026, HMMT November 2025, and HMMT February 2026.
All three postdate the release of every student model and are therefore absent from its training data, including undisclosed post-training data.
We report Avg@32 on all benchmarks, define validation accuracy as the mean over AIME 2024 and AIME 2025, and give further training hyperparameter details in App.~\ref{app:training_details}.
\subsection{Masking Low-Divergence States Improves Transfer}
\label{sec:main_results}

\textbf{Retaining 10\% of states improves held-out transfer.}
Tab.~\ref{tab:main_result} compares S$^2$D-OPD with dense Direct-OPD across four students and two teacher pairs. 
Retaining 10\% of states achieves comparable validation accuracy and improves the held-out Test Avg.\ in seven of eight settings, with a tie in the eighth. 
We assess statistical significance over 93 held-out problems, first averaging each problem’s 32 responses and then averaging its paired differences across the eight settings.
Using a paired problem bootstrap and an exact one-sided sign-flip test, we find that S$^2$D-OPD improves mean held-out accuracy from 46.36\% to 47.31\%, a gain of 0.95 points (95\% CI: 0.40--1.54; $p=6.4\times10^{-4}$).
\begin{figure}[t]
    \centering
    \includegraphics[width=0.9\linewidth]{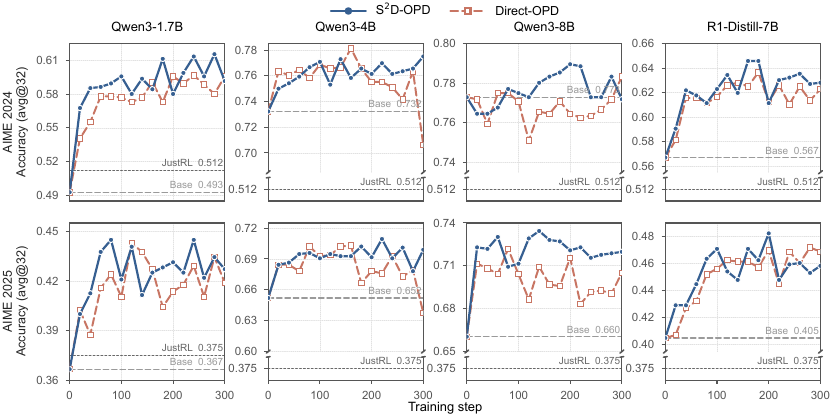}
    \caption{Validation accuracy (Avg@32) under the JustRL teacher pair across four students. Dashed lines mark the initial student; dotted lines mark the JustRL-1.5B teacher.}
    \label{fig:main_result}
\end{figure}
\begin{figure}[tbp]
    \centering
    \begin{subfigure}[t]{0.640\linewidth}
        \centering
        \includegraphics[width=\linewidth]{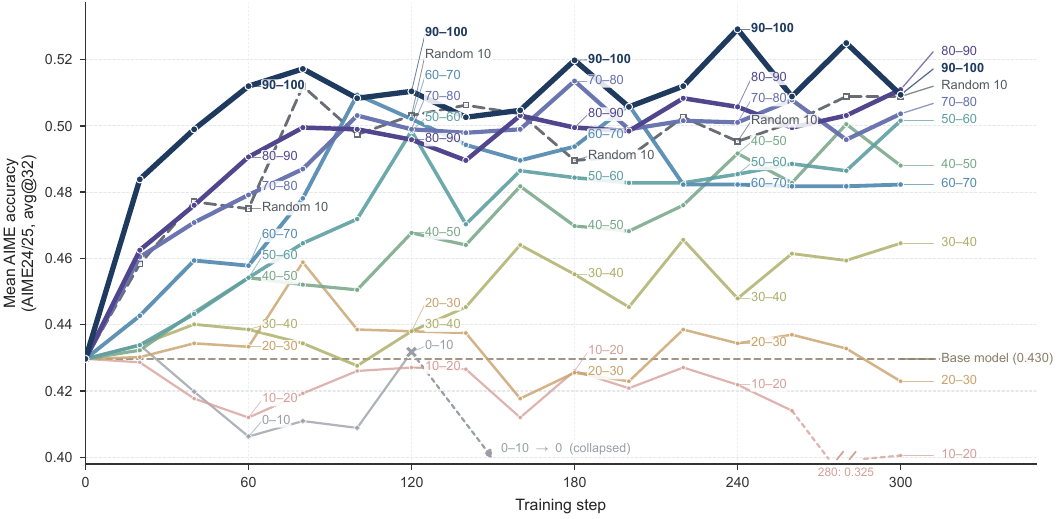}
        \caption{Performance across JSD percentile bins.}
        \label{fig:ladder_avg_32}
    \end{subfigure}
    \hfill
    \begin{subfigure}[t]{0.330\linewidth}
        \centering
        \includegraphics[width=\linewidth]{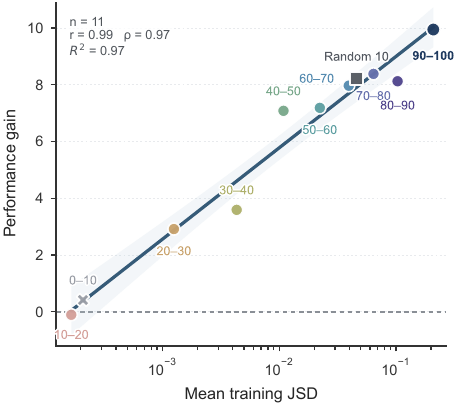}
        \caption{Mean JSD tracks transfer gains.}
        \label{fig:jsd_best_gain}
    \end{subfigure}
    \caption{
    Which states are retained, not how many, determines transfer (Qwen3-1.7B, JustRL teacher pair).
    \textbf{(a)} Validation accuracy when training on each JSD percentile bin or on a random 10\% of states.
    \textbf{(b)} Gain in peak validation accuracy over the initial student versus the mean training JSD of the retained states; the line is a log-linear fit.
}
    \label{fig:jsd_supervision_quality}
\end{figure}

\textbf{Selective supervision preserves early learning while improving
late-stage stability.}
Fig.~\ref{fig:main_result} shows the validation trajectories under the JustRL teacher pair.
Despite masking low-divergence states, S$^2$D-OPD keeps pace with dense Direct-OPD during early training, matching or exceeding its accuracy over the first 60 steps for all four students.
Discarding 90\% of states thus does not slow early learning, suggesting that the supervision driving early improvement is concentrated in the retained high-divergence subset.
Later in training, S$^2$D-OPD maintains higher accuracy for all four students and is more stable on Qwen3-4B and Qwen3-8B, where dense Direct-OPD repeatedly falls back from its peaks and, on Qwen3-4B, even drops below the initial student by step 300.
These results answer RQ1: masking low-JSD states improves held-out transfer without slowing early learning.

\subsection{Low-Divergence Supervision Is Redundant and Harmful}
\label{sec:section_ladder}
Sec.~\ref{sec:MethodMotivation} shows that low-JSD states are those at which the teacher's behavior changed little, yet the Direct-OPD update there can be as large as anywhere else.
To test what such supervision contributes, we fix the JustRL teacher pair, the Qwen3-1.7B student, the training data, and the hyperparameters, and vary only which states enter the Direct-OPD objective.
Within each response, we rank states by JSD, split them into ten equal-sized percentile bins (0--10, \ldots, 90--100), and train a separate student on each bin.
As a control, we train on a uniformly sampled 10\% of states, whose update is an unbiased estimate of the dense Direct-OPD update (Fig.~\ref{fig:ladder_avg_32}).

\textbf{Dense supervision is largely redundant.}
The random 10\% control reaches a peak validation accuracy of 51.2, on par with dense Direct-OPD (51.3; Tab.~\ref{tab:main_result}).
Discarding 90\% of states at random thus loses little, supporting the view of~\citet{rethinkopd2} that OPD is ``data-overfed but algorithm-starved.''

\textbf{Which states are retained determines transfer.}
At the same 10\% budget, performance rises broadly with JSD percentile, and the top bin reaches 52.9, above both the random control and dense Direct-OPD.
The effect is graded: across all eleven subsets, the gain in peak validation accuracy over the initial student is approximately linear in the logarithm of the subset's mean training JSD ($R^2=0.97$; Fig.~\ref{fig:jsd_best_gain}).
The random control lies on the same line, so within this setting the mean divergence of the retained states predicts transfer whether they are selected by rank or at random.

\textbf{Low-divergence supervision is harmful, not merely weak.}
Trained in isolation, the three lowest bins end below the initial student, and the two lowest collapse.
This matches the regime identified in Sec.~\ref{sec:MethodMotivation}, where the Direct-OPD update persists although the teacher barely changed.
Top-$K$ overlap suggests a mechanism (App.~\ref{app:token_overlap}): the top bin moves the student toward the teacher while preserving its overlap with the reference, whereas lower bins move it away from the reference without a commensurate gain in teacher alignment.

\textbf{Divergence and reward magnitude rank states differently.}
Fig.~\ref{fig:mask-residue-main} shows a state from an initial Qwen3-1.7B rollout under the JustRL pair.
RL lowers the probability of the sampled token \texttt{5} from 0.12 to $2\times10^{-5}$, giving this state the largest Direct-OPD reward magnitude among the response's 8{,}192 states, although the teacher's distribution moves by only 0.12 in total variation.
Its JSD (0.044) falls below the retention cutoff (0.067), so S$^2$D-OPD masks it while retaining a state with a smaller log-ratio but a larger redistribution of mass (JSD 0.262).

\begin{figure}[tbp]
    \centering
    \begin{subfigure}[t]{0.585\linewidth}
        \centering
        \includegraphics[width=\linewidth]{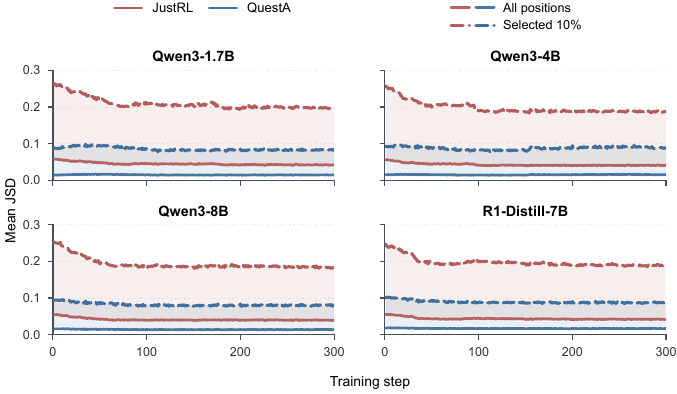}
        \caption{All-position and selected-position JSD.}
        \label{fig:js_mean_dynamics}
    \end{subfigure}
    \hfill
    \begin{subfigure}[t]{0.405\linewidth}
        \centering
        \includegraphics[width=\linewidth]{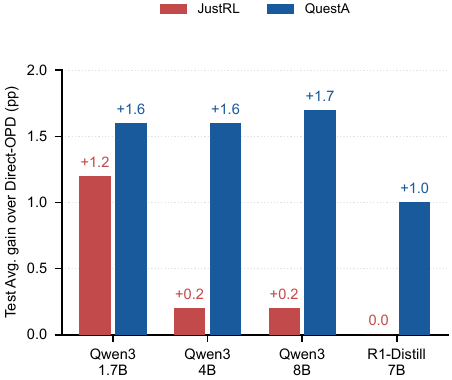}
        \caption{Held-out gains over Direct-OPD.}
        \label{fig:js_mean_performance}
    \end{subfigure}
    \caption{
        Teacher-pair comparison with top-10\% retention across four students.
        \textbf{(a)} Mean teacher--reference JSD over all valid response
        states (solid) and retained states (dashed) during
        S$^2$D-OPD training.
        \textbf{(b)} Held-out Test Avg. gains of S$^2$D-OPD
        over dense Direct-OPD under each teacher pair, in percentage points.
    }
    \label{fig:js_select}
\end{figure}

\begin{figure}[t]
    \centering
    \includegraphics[width=0.92\linewidth]{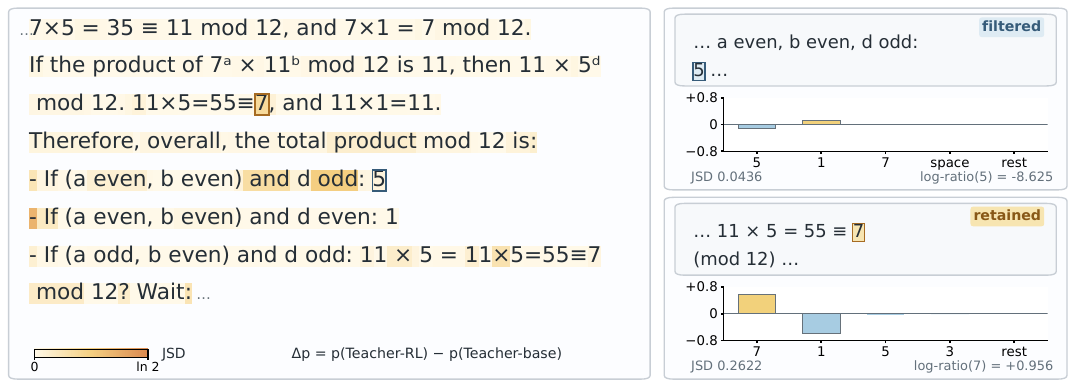}
    \caption{
    \textbf{JSD and token log-ratio rank states differently.}
    The state with the largest sampled-token log-ratio in the response is masked, while a larger redistribution of probability mass is retained.
    }
    \label{fig:mask-residue-main}
\end{figure}

\textbf{Masking helps more under the lower-divergence teacher pair.}
Mean JSD is higher for JustRL than for QuestA over both all states and the retained top 10\% (Fig.~\ref{fig:js_mean_dynamics}), yet the held-out gains of masking over dense Direct-OPD are larger under QuestA: 1.0--1.7 points versus 0.0--1.2 for JustRL across the four students (Fig.~\ref{fig:js_mean_performance}).
If the benefit came from concentrating supervision on high absolute divergence, this ordering would be reversed; together with the harm of low-JSD bins, it suggests that masking helps mainly by removing supervision that undermines learning.
These results answer RQ2: JSD ranks supervision utility within a teacher--student setting, and masking pays off by filtering low-divergence supervision rather than by maximizing the divergence of what remains.

\subsection{Selective Masking Is Robust Across Design Choices}\label{sec:extra_exp}
We vary three design choices while keeping the rest of S$^2$D-OPD fixed: the divergence used for scoring, the scope over which states are ranked, and the retention rate (Fig.~\ref{fig:combined_ablation}).

\textbf{Divergence measure.}
The construction in Sec.~\ref{sec:MethodMotivation} shows that both KL directions vanish with the candidate mass just as JSD does, so either should flag the same low-mass states.
Under the JustRL pair, selection by forward KL, $D_{\mathrm{KL}}(\pi_{\mathrm T}\Vert\pi_{\mathrm{ref}})$, or reverse KL, $D_{\mathrm{KL}}(\pi_{\mathrm{ref}}\Vert\pi_{\mathrm T})$, attains higher mean validation accuracy than dense Direct-OPD from step 160 onward on all three Qwen3 students, and avoids the late decline of dense Direct-OPD on Qwen3-4B (final accuracy 71.6--74.1 versus 67.2).
Forward KL closely tracks JSD at all three scales, whereas reverse KL is weaker on Qwen3-4B.

\textbf{Selection scope.}
Ranking states across the whole batch, so that a response may retain none, matches response-level selection on Qwen3-1.7B under both teacher pairs (peak validation accuracy 52.3 versus 52.9 under JustRL and 51.7 versus 51.7 under QuestA).
The per-response rule of Sec.~\ref{sec:methods} is thus not needed for accuracy; we keep it so that each mask is independent of other responses.

\textbf{Retention rate.}
Varying $\rho$ from 5\% to 20\% on Qwen3-1.7B changes peak validation accuracy by less than one point (52.2--53.0), and every rate stays above dense Direct-OPD (51.3) and the random 10\% control (51.2; App.~\ref{app:threshold_sensitivity}).
These results answer RQ3: within the tested range, the gains do not hinge on the divergence measure, selection scope, or retention rate.

\begin{figure*}[tbp]
    \centering
\begin{subfigure}[t]{0.595\textwidth}
    \vspace{0pt}
    \centering
    \includegraphics[width=\linewidth]{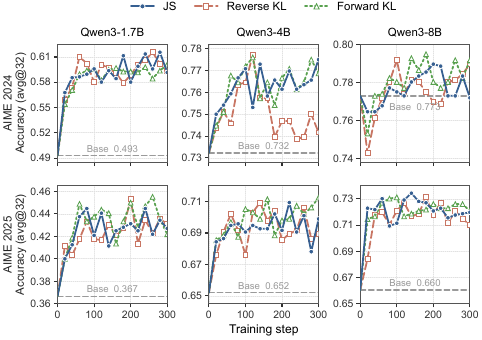}
    \caption{Divergence measure.}
    \label{fig:kl}
\end{subfigure}
\hfill
\begin{subfigure}[t]{0.390\textwidth}
    \vspace{0pt}
    \centering
    \includegraphics[width=\linewidth]{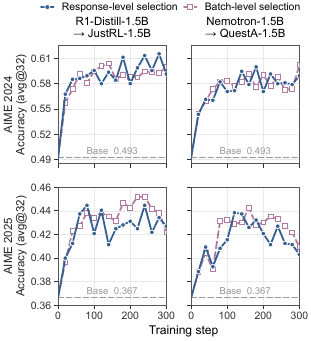}
    \caption{Selection scope.}
    \label{fig:aggregation}
\end{subfigure}
    \caption{
        Robustness of selective masking across design choices.
        \textbf{(a)} Top-10\% selection by JSD, reverse KL, and forward KL on Qwen3-1.7B, 4B, and 8B under the JustRL teacher pair.
        \textbf{(b)} Response-level versus batch-level selection on Qwen3-1.7B under both teacher pairs. 
    }
    \label{fig:combined_ablation}
\end{figure*}


\section{Conclusion}
Direct-OPD transfers the policy shift that RL induces in a small teacher to a larger student by rewarding every student-sampled state with the teacher--reference log-ratio.
We showed that this reward ignores the probability mass behind the shift: it can stay fixed while the mass on the student's candidates, and with it the teacher--reference divergence, vanishes.
Building on this observation, we introduced S$^2$D-OPD, which retained Direct-OPD supervision only at the highest-JSD states of each response and required no extra forward passes.
Retaining 10\% of states improved held-out accuracy over dense Direct-OPD in seven of eight teacher--student settings and tied in the eighth.
Controlled analyses further showed that dense supervision was largely redundant, that transfer depended on which states were retained rather than how many, and that low-divergence supervision could harm the student when trained on alone; the gains persisted across divergence measures, selection scopes, and retention rates.
These findings suggest that weak-to-strong policy transfer depends not only on what is transferred but also on where it is applied.
Extending state selection to adaptive retention rates and to domains beyond mathematical reasoning is a natural next step.

\subsection*{AI Use Statement}

We used generative AI tools solely to improve the readability and
clarity of the manuscript, including language editing and suggestions
on presentation. All AI-assisted revisions were reviewed by the
authors, who take full responsibility for the final content of
this work.


\subsection*{Reproducibility Statement}

To facilitate reproducibility, we describe the selective supervision
procedure in Section~\ref{sec:method} and report the teacher--reference
pairs, datasets, training configurations, and evaluation protocols in
Section~\ref{sec:experiment_setting} and
Appendix~\ref{app:training_details}.
Our code is fully disclosed on GitHub.

\bibliography{iclr2027_conference}

@misc{r3,
      title={Stabilizing MoE Reinforcement Learning by Aligning Training and Inference Routers}, 
      author={Wenhan Ma and Hailin Zhang and Liang Zhao and Yifan Song and Yudong Wang and Zhifang Sui and Fuli Luo},
      year={2025},
      eprint={2510.11370},
      archivePrefix={arXiv},
      primaryClass={cs.CL},
      url={https://arxiv.org/abs/2510.11370}, 
}

@misc{slime,
  author       = {Zilin Zhu and Chengxing Xie and Xin Lv and slime Contributors},
  title        = {slime: An LLM post-training framework for RL Scaling},
  year         = {2025},
  howpublished = {\url{https://github.com/THUDM/slime}},
  note         = {GitHub repository. Corresponding author: Xin Lv},
  urldate      = {2025-06-19}
}

@misc{sao,
      title={Single-Rollout Asynchronous Optimization for Agentic Reinforcement Learning}, 
      author={Zhenyu Hou and Yujiang Li and Jie Tang and Yuxiao Dong},
      year={2026},
      eprint={2607.07508},
      archivePrefix={arXiv},
      primaryClass={cs.LG},
      url={https://arxiv.org/abs/2607.07508}, 
}

@misc{kimi,
      title={Kimi K2.5: Visual Agentic Intelligence}, 
      author={Kimi Team and Tongtong Bai and Yifan Bai and Yiping Bao and S. H. Cai and Yuan Cao and Y. Charles and H. S. Che and Cheng Chen and Guanduo Chen and Huarong Chen and Jia Chen and Jiahao Chen and Jianlong Chen and Jun Chen and Kefan Chen and others},
      year={2026},
      eprint={2602.02276},
      archivePrefix={arXiv},
      primaryClass={cs.CL},
      url={https://arxiv.org/abs/2602.02276}, 
}

@misc{deepseek,
      title={DeepSeek-V4: Towards Highly Efficient Million-Token Context Intelligence}, 
      author={DeepSeek-AI and Anyi Xu and Bangcai Lin and Bing Xue and Bingxuan Wang and Bingzheng Xu and Bochao Wu and Bowei Zhang and Chaofan Lin and Chen Dong and Chenchen Ling and Chengda Lu and Chenggang Zhao and Chengqi Deng and Chengyu Hou and Chenhao Xu and others},
      year={2026},
      eprint={2606.19348},
      archivePrefix={arXiv},
      primaryClass={cs.CL},
      url={https://arxiv.org/abs/2606.19348}, 
}

@misc{glm5,
      title={GLM-5: from Vibe Coding to Agentic Engineering}, 
      author={GLM-5-Team and Aohan Zeng and Xin Lv and Zhenyu Hou and Zhengxiao Du and Qinkai Zheng and Bin Chen and Da Yin and Chendi Ge and Chenghua Huang and Chengxing Xie and Chenzheng Zhu and Congfeng Yin and Cunxiang Wang and Gengzheng Pan and Hao Zeng and others},
      year={2026},
      eprint={2602.15763},
      archivePrefix={arXiv},
      primaryClass={cs.LG},
      url={https://arxiv.org/abs/2602.15763}, 
}

@misc{scalinglaw1,
      title={Scaling Laws for Neural Language Models}, 
      author={Jared Kaplan and Sam McCandlish and Tom Henighan and Tom B. Brown and Benjamin Chess and Rewon Child and Scott Gray and Alec Radford and Jeffrey Wu and Dario Amodei},
      year={2020},
      eprint={2001.08361},
      archivePrefix={arXiv},
      primaryClass={cs.LG},
      url={https://arxiv.org/abs/2001.08361}, 
}

@inproceedings{scalinglaw2,
author = {Hoffmann, Jordan and Borgeaud, Sebastian and Mensch, Arthur and Buchatskaya, Elena and Cai, Trevor and Rutherford, Eliza and de Las Casas, Diego and Hendricks, Lisa Anne and Welbl, Johannes and Clark, Aidan and Hennigan, Tom and Noland, Eric and Millican, Katie and van den Driessche, George and Damoc, Bogdan and others},
title = {Training compute-optimal large language models},
year = {2022},
isbn = {9781713871088},
publisher = {Curran Associates Inc.},
address = {Red Hook, NY, USA},
booktitle = {Proceedings of the 36th International Conference on Neural Information Processing Systems},
articleno = {2176},
numpages = {15},
location = {New Orleans, LA, USA},
series = {NIPS '22}
}

@article{
scalinglaw3,
title={Reconciling Kaplan and Chinchilla Scaling Laws},
author={Tim Pearce and Jinyeop Song},
journal={Transactions on Machine Learning Research},
issn={2835-8856},
year={2024},
url={https://openreview.net/forum?id=NLoaLyuUUF},
note={Reproducibility Certification}
}

@inproceedings{
rl1,
title={Does Reinforcement Learning Really Incentivize Reasoning Capacity in {LLM}s Beyond the Base Model?},
author={Yang Yue and Zhiqi Chen and Rui Lu and Andrew Zhao and Zhaokai Wang and Yang Yue and Shiji Song and Gao Huang},
booktitle={The Thirty-ninth Annual Conference on Neural Information Processing Systems},
year={2025},
url={https://openreview.net/forum?id=4OsgYD7em5}
}

@misc{grpo,
      title={DeepSeekMath: Pushing the Limits of Mathematical Reasoning in Open Language Models}, 
      author={Zhihong Shao and Peiyi Wang and Qihao Zhu and Runxin Xu and Junxiao Song and Xiao Bi and Haowei Zhang and Mingchuan Zhang and Y. K. Li and Y. Wu and Daya Guo},
      year={2024},
      eprint={2402.03300},
      archivePrefix={arXiv},
      primaryClass={cs.CL},
      url={https://arxiv.org/abs/2402.03300}, 
}

@inproceedings{
rl3,
title={Echo Chamber: {RL} Post-training Amplifies Behaviors Learned in Pretraining},
author={Rosie Zhao and Alexandru Meterez and Sham M. Kakade and Cengiz Pehlevan and Samy Jelassi and Eran Malach},
booktitle={Second Conference on Language Modeling},
year={2025},
url={https://openreview.net/forum?id=dp4KWuSDzj}
}

@misc{llamarl,
      title={LlamaRL: A Distributed Asynchronous Reinforcement Learning Framework for Efficient Large-scale LLM Training}, 
      author={Bo Wu and Sid Wang and Yunhao Tang and Jia Ding and Eryk Helenowski and Liang Tan and Tengyu Xu and Tushar Gowda and Zhengxing Chen and Chen Zhu and Xiaocheng Tang and Yundi Qian and Beibei Zhu and Rui Hou},
      year={2025},
      eprint={2505.24034},
      archivePrefix={arXiv},
      primaryClass={cs.LG},
      url={https://arxiv.org/abs/2505.24034}, 
}

@inproceedings{instable,
title={{ARLA}rena: A Unified Framework for Stable Agentic Reinforcement Learning},
author={Xiaoxuan Wang and Han Zhang and Haixin Wang and Yidan Shi and Ruoyan Li and Kaiqiao Han and Chenyi Tong and Haoran Deng and Alexander K Taylor and Renliang Sun and Yanqiao Zhu and Jason Cong and Yizhou Sun and Wei Wang},
booktitle={Forty-third International Conference on Machine Learning},
year={2026},
url={https://openreview.net/forum?id=90kxFi9VGP}
}

@misc{dapo,
      title={DAPO: An Open-Source LLM Reinforcement Learning System at Scale}, 
      author={Qiying Yu and Zheng Zhang and Ruofei Zhu and Yufeng Yuan and Xiaochen Zuo and Yu Yue and Weinan Dai and Tiantian Fan and Gaohong Liu and Lingjun Liu and Xin Liu and Haibin Lin and Zhiqi Lin and Bole Ma and Guangming Sheng and Yuxuan Tong and others},
      year={2025},
      eprint={2503.14476},
      archivePrefix={arXiv},
      primaryClass={cs.LG},
      url={https://arxiv.org/abs/2503.14476}, 
}

@misc{gspo,
      title={Group Sequence Policy Optimization}, 
      author={Chujie Zheng and Shixuan Liu and Mingze Li and Xiong-Hui Chen and Bowen Yu and Chang Gao and Kai Dang and Yuqiong Liu and Rui Men and An Yang and Jingren Zhou and Junyang Lin},
      year={2025},
      eprint={2507.18071},
      archivePrefix={arXiv},
      primaryClass={cs.LG},
      url={https://arxiv.org/abs/2507.18071}, 
}

@misc{cispo,
      title={MiniMax-M1: Scaling Test-Time Compute Efficiently with Lightning Attention}, 
      author={MiniMax and Aili Chen and Aonian Li and Bangwei Gong and Binyang Jiang and Bo Fei and Bo Yang and Boji Shan and Changqing Yu and Chao Wang and Cheng Zhu and Chengjun Xiao and Chengyu Du and Chi Zhang and Chu Qiao and Chunhao Zhang and others},
      year={2025},
      eprint={2506.13585},
      archivePrefix={arXiv},
      primaryClass={cs.CL},
      url={https://arxiv.org/abs/2506.13585}, 
}

@inproceedings{verl,
author = {Sheng, Guangming and Zhang, Chi and Ye, Zilingfeng and Wu, Xibin and Zhang, Wang and Zhang, Ru and Peng, Yanghua and Lin, Haibin and Wu, Chuan},
title = {HybridFlow: A Flexible and Efficient RLHF Framework},
year = {2025},
isbn = {9798400711961},
publisher = {Association for Computing Machinery},
address = {New York, NY, USA},
url = {https://doi.org/10.1145/3689031.3696075},
doi = {10.1145/3689031.3696075},
booktitle = {Proceedings of the Twentieth European Conference on Computer Systems},
pages = {1279–1297},
numpages = {19},
location = {Rotterdam, Netherlands},
series = {EuroSys '25}
}

@inproceedings{areal,
title={{AREAL}: A Large-Scale Asynchronous Reinforcement Learning System for Language Reasoning},
author={Wei Fu and Jiaxuan Gao and Xujie Shen and Chen Zhu and Zhiyu Mei and Chuyi He and Shusheng Xu and Guo Wei and Jun Mei and WANG JIASHU and Tongkai Yang and Binhang Yuan and Yi Wu},
booktitle={The Thirty-ninth Annual Conference on Neural Information Processing Systems},
year={2025},
url={https://openreview.net/forum?id=X9diEuva9R}
}

@inproceedings{vllm,
author = {Kwon, Woosuk and Li, Zhuohan and Zhuang, Siyuan and Sheng, Ying and Zheng, Lianmin and Yu, Cody Hao and Gonzalez, Joseph and Zhang, Hao and Stoica, Ion},
title = {Efficient Memory Management for Large Language Model Serving with PagedAttention},
year = {2023},
isbn = {9798400702297},
publisher = {Association for Computing Machinery},
address = {New York, NY, USA},
url = {https://doi.org/10.1145/3600006.3613165},
doi = {10.1145/3600006.3613165},
booktitle = {Proceedings of the 29th Symposium on Operating Systems Principles},
pages = {611–626},
numpages = {16},
location = {Koblenz, Germany},
series = {SOSP '23}
}

@inproceedings{megatron,
author = {Narayanan, Deepak and Shoeybi, Mohammad and Casper, Jared and LeGresley, Patrick and Patwary, Mostofa and Korthikanti, Vijay and Vainbrand, Dmitri and Kashinkunti, Prethvi and Bernauer, Julie and Catanzaro, Bryan and Phanishayee, Amar and Zaharia, Matei},
title = {Efficient large-scale language model training on GPU clusters using megatron-LM},
year = {2021},
isbn = {9781450384421},
publisher = {Association for Computing Machinery},
address = {New York, NY, USA},
url = {https://doi.org/10.1145/3458817.3476209},
doi = {10.1145/3458817.3476209},
booktitle = {Proceedings of the International Conference for High Performance Computing, Networking, Storage and Analysis},
articleno = {58},
numpages = {15},
location = {St. Louis, Missouri},
series = {SC '21}
}

@misc{rethinkopd,
      title={Rethinking On-Policy Distillation of Large Language Models: Phenomenology, Mechanism, and Recipe}, 
      author={Yaxuan Li and Yuxin Zuo and Bingxiang He and Jinqian Zhang and Chaojun Xiao and Cheng Qian and Tianyu Yu and Huan-ang Gao and Wenkai Yang and Zhiyuan Liu and Ning Ding},
      year={2026},
      eprint={2604.13016},
      archivePrefix={arXiv},
      primaryClass={cs.LG},
      url={https://arxiv.org/abs/2604.13016}, 
}

@misc{direct-opd,
      title={Weak-to-Strong Generalization via Direct On-Policy Distillation}, 
      author={Shiyuan Feng and Huan-ang Gao and Haohan Chi and Hanlin Wu and Zhilong Zhang and Zheng Jiang and Bingxiang He and Wei-Ying Ma and Ya-Qin Zhang and Hao Zhou},
      year={2026},
      eprint={2607.05394},
      archivePrefix={arXiv},
      primaryClass={cs.LG},
      url={https://arxiv.org/abs/2607.05394}, 
}

@misc{ProxyOPD,
      title={Proxy OPD: On-Policy Distillation with Transferable Relative Proxy Update}, 
      author={Daocheng Fu and Rong Wu and Yu Yang and Jianbiao Mei and Licheng Wen and Pinlong Cai and Xuemeng Yang and Yong Liu and Botian Shi and Yu Qiao},
      year={2026},
      eprint={2607.11505},
      archivePrefix={arXiv},
      primaryClass={cs.LG},
      url={https://arxiv.org/abs/2607.11505}, 
}

@misc{Exopd,
      title={Learning beyond Teacher: Generalized On-Policy Distillation with Reward Extrapolation}, 
      author={Wenkai Yang and Weijie Liu and Ruobing Xie and Kai Yang and Saiyong Yang and Yankai Lin},
      year={2026},
      eprint={2602.12125},
      archivePrefix={arXiv},
      primaryClass={cs.LG},
      url={https://arxiv.org/abs/2602.12125}, 
}

@misc{Ta-opd,
      title={Not All Disagreement Is Learnable: Token Teachability in On-Policy Distillation}, 
      author={Yuanyi Wang and Su Lu and Yanggan Gu and Pengkai Wang and Yifan Yang and Zhaoyi Yan and Congkai Xie and Jianmin Wu and Hongxia Yang},
      year={2026},
      eprint={2605.26844},
      archivePrefix={arXiv},
      primaryClass={cs.LG},
      url={https://arxiv.org/abs/2605.26844}, 
}

@misc{opdd,
      title={On-Policy Delta Distillation}, 
      author={Byeongho Heo and Jaehui Hwang and Sangdoo Yun and Dongyoon Han},
      year={2026},
      eprint={2607.15161},
      archivePrefix={arXiv},
      primaryClass={cs.LG},
      url={https://arxiv.org/abs/2607.15161}, 
}

@misc{tip,
      title={TIP: Token Importance in On-Policy Distillation}, 
      author={Yuanda Xu and Hejian Sang and Zhengze Zhou and Ran He and Zhipeng Wang and Alborz Geramifard},
      year={2026},
      eprint={2604.14084},
      archivePrefix={arXiv},
      primaryClass={cs.LG},
      url={https://arxiv.org/abs/2604.14084}, 
}

@article{opd,
  author = {Kevin Lu and Thinking Machines Lab},
  title = {On-Policy Distillation},
  journal = {Thinking Machines Lab: Connectionism},
  year = {2025},
  note = {https://thinkingmachines.ai/blog/on-policy-distillation},
  doi = {10.64434/tml.20251026},
}

@inproceedings{
GKD,
title={On-Policy Distillation of Language Models: Learning from Self-Generated Mistakes},
author={Rishabh Agarwal and Nino Vieillard and Yongchao Zhou and Piotr Stanczyk and Sabela Ramos Garea and Matthieu Geist and Olivier Bachem},
booktitle={The Twelfth International Conference on Learning Representations},
year={2024},
url={https://openreview.net/forum?id=3zKtaqxLhW}
}

@inproceedings{
minillm,
title={Mini{LLM}: Knowledge Distillation of Large Language Models},
author={Yuxian Gu and Li Dong and Furu Wei and Minlie Huang},
booktitle={The Twelfth International Conference on Learning Representations},
year={2024},
url={https://openreview.net/forum?id=5h0qf7IBZZ}
}

@misc{REOPOLD,
      title={Scaling Reasoning Efficiently via Relaxed On-Policy Distillation}, 
      author={Jongwoo Ko and Sara Abdali and Young Jin Kim and Tianyi Chen and Pashmina Cameron},
      year={2026},
      eprint={2603.11137},
      archivePrefix={arXiv},
      primaryClass={cs.LG},
      url={https://arxiv.org/abs/2603.11137}, 
}

@misc{SE-KD,
      title={Rethinking Selective Knowledge Distillation}, 
      author={Almog Tavor and Itay Ebenspanger and Neil Cnaan and Mor Geva},
      year={2026},
      eprint={2602.01395},
      archivePrefix={arXiv},
      primaryClass={cs.CL},
      url={https://arxiv.org/abs/2602.01395}, 
}

@inproceedings{
TeacherSide1,
title={Entropy-Aware On-Policy Distillation of Language Models},
author={Woogyeol Jin and Taywon Min and Yongjin Yang and Dennis Wei and Yi Zhou and Swanand Ravindra Kadhe and Nathalie Baracaldo and Kimin Lee},
booktitle={Forty-third International Conference on Machine Learning},
year={2026},
url={https://openreview.net/forum?id=J5i09faOOf}
}

@misc{TeacherSide2,
      title={SAGE-OPD: Selective Agent-Guided Intervention for Multi-Turn On-Policy Distillation}, 
      author={Yuhang Zhou and Lizhu Zhang and Yifan Wu and Mingyi Wang and Bo Peng and Jiayi Liu and Xiangjun Fan and Zhuokai Zhao},
      year={2026},
      eprint={2606.19659},
      archivePrefix={arXiv},
      primaryClass={cs.CL},
      url={https://arxiv.org/abs/2606.19659}, 
}

@misc{opsd,
      title={Self-Distilled Reasoner: On-Policy Self-Distillation for Large Language Models}, 
      author={Siyan Zhao and Zhihui Xie and Mengchen Liu and Jing Huang and Guan Pang and Feiyu Chen and Aditya Grover},
      year={2026},
      eprint={2601.18734},
      archivePrefix={arXiv},
      primaryClass={cs.LG},
      url={https://arxiv.org/abs/2601.18734}, 
}

@misc{revisitopd,
      title={Revisiting On-Policy Distillation: Empirical Failure Modes and Simple Fixes}, 
      author={Yuqian Fu and Haohuan Huang and Kaiwen Jiang and Jiacai Liu and Zhuo Jiang and Yuanheng Zhu and Dongbin Zhao},
      year={2026},
      eprint={2603.25562},
      archivePrefix={arXiv},
      primaryClass={cs.LG},
      url={https://arxiv.org/abs/2603.25562}, 
}

@misc{Aopd,
      title={Asymmetric On-Policy Distillation: Bridging Exploitation and Imitation at the Token Level}, 
      author={Nan Jia and Haojin Yang and Xing Ma and Jiesong Lian and Shuailiang Zhang and Weipeng Zhang and Ke Zeng and Xunliang Cai and Zequn Sun},
      year={2026},
      eprint={2605.06387},
      archivePrefix={arXiv},
      primaryClass={cs.LG},
      url={https://arxiv.org/abs/2605.06387}, 
}

@misc{RLCSD,
      title={RLCSD: Reinforcement Learning with Contrastive On-Policy Self-Distillation}, 
      author={Leyi Pan and Shuchang Tao and Yunpeng Zhai and Lingzhe Zhang and Zhaoyang Liu and Bolin Ding and Aiwei Liu and Lijie Wen},
      year={2026},
      eprint={2606.11709},
      archivePrefix={arXiv},
      primaryClass={cs.LG},
      url={https://arxiv.org/abs/2606.11709}, 
}

@inproceedings{
sparsebutcritical,
title={Sparse but Critical: A Token-Level Analysis of Distributional Shifts in {RLVR} Fine-Tuning of {LLM}s},
author={Haoming Meng and Kexin Huang and Shaohang Wei and Chiyu Ma and Shuo Yang and Xue Wang and Guoyin Wang and Bolin Ding and Jingren Zhou},
booktitle={The Fourteenth International Conference on Learning Representations},
year={2026},
url={https://openreview.net/forum?id=8vWIXno8LW}
}

@misc{skywork_dataset,
      title={Skywork Open Reasoner 1 Technical Report}, 
      author={Jujie He and Jiacai Liu and Chris Yuhao Liu and Rui Yan and Chaojie Wang and Peng Cheng and Xiaoyu Zhang and Fuxiang Zhang and Jiacheng Xu and Wei Shen and Siyuan Li and Liang Zeng and Tianwen Wei and Cheng Cheng and Bo An and Yang Liu and Yahui Zhou},
      year={2025},
      eprint={2505.22312},
      archivePrefix={arXiv},
      primaryClass={cs.LG},
      url={https://arxiv.org/abs/2505.22312}, 
}

@misc{justRL,
      title={JustRL: Scaling a 1.5B LLM with a Simple RL Recipe}, 
      author={Bingxiang He and Zekai Qu and Zeyuan Liu and Yinghao Chen and Yuxin Zuo and Cheng Qian and Kaiyan Zhang and Weize Chen and Chaojun Xiao and Ganqu Cui and Ning Ding and Zhiyuan Liu},
      year={2025},
      eprint={2512.16649},
      archivePrefix={arXiv},
      primaryClass={cs.CL},
      url={https://arxiv.org/abs/2512.16649}, 
}

@inproceedings{
questa,
title={QuestA: Expanding Reasoning Capacity in {LLM}s via Question Augmentation},
author={Jiazheng Li and Hongzhou Lin and Hong Lu and Kaiyue Wen and Zaiwen Yang and Jiaxuan Gao and Yi Wu and Jingzhao Zhang},
booktitle={The Fourteenth International Conference on Learning Representations},
year={2026},
url={https://openreview.net/forum?id=3MifB0f7qR}
}

@misc{rethinkopd2,
      title={Rethinking On-Policy Distillation of Large Language Models II: One Training Example}, 
      author={Zixuan Fu and Bingxiang He and Yuxin Zuo and Haohuan Huang and Jinqian Zhang and Ruhang Xiao and Cheng Qian and Qinyu Luo and Huan-ang Gao and Yudong Wang and Zhiyuan Liu and Ning Ding and Chaojun Xiao},
      year={2026},
      eprint={2609.04172},
      archivePrefix={arXiv},
      primaryClass={cs.AI},
      url={https://arxiv.org/abs/2609.04172}, 
}

@inproceedings{
cmc,
title={Cross-model Control: Improving Multiple Large Language Models in One-time Training},
author={Jiayi Wu and Hao Sun and Hengyi Cai and Lixin Su and Shuaiqiang Wang and Dawei Yin and Xiang Li and Ming Gao},
booktitle={The Thirty-eighth Annual Conference on Neural Information Processing Systems},
year={2024},
url={https://openreview.net/forum?id=YPqHSTSoFs}
}
\bibliographystyle{iclr2027_conference}

\appendix

\section{Optimization Details}
\label{app:selective_DOPD_Objective}

This section details the adaptive KL control inherited from Direct-OPD and the selective position-wise aggregation used in S$^2$D-OPD.

\subsection{Adaptive KL Control}
\label{app:direct_opd_optimization}

The KL coefficient controls the strength of the student anchor relative to the policy-shift reward. We update this coefficient using Direct-OPD's sign-based controller:
\begin{equation}
\label{eq:adaptive_kl}
\alpha_{m+1}
=
\operatorname{clip}\!\left(
\alpha_m \left[1+\epsilon\,\operatorname{sgn}(\bar r_m)\right],
\alpha_{\min},\alpha_{\max}
\right),
\end{equation}
where $\bar r_m$ denotes the mean student-weighted policy shift $\bar p_t(v)\Delta_t(v\mid\vs_t)$ over valid response positions and their top-$K$ candidates at iteration $m$.
The controller increases $\alpha$ when the mean weighted policy shift is positive and decreases it when negative, subject to the bounds.
We use $\alpha_0=2.5$, $\epsilon=0.01$, and $[\alpha_{\min},\alpha_{\max}]=[0.5,2.5]$.
The updated coefficient $\alpha_{m+1}$ is used in the subsequent actor update.

\subsection{Detailed Loss Aggregation}
\label{app:aggregation}

Direct-OPD averages the local update over all valid response positions. 
S$^2$D-OPD keeps the local update unchanged but restricts this average to selected positions.
Let $\mathcal I$ collect the positions selected independently within each response by Eq.~\ref{eq:selection} across the batch.
Giving each retained position equal weight yields the local ascent direction
\begin{equation}
\label{eq:full_gradient}
g_\theta^{\mathrm{S^2D\text{-}OPD}}
=
\frac{1}{|\mathcal I|}
\sum_{t\in\mathcal I}
\Biggl[
\sum_{v\in\mathcal V_K(\vs_t)}
\operatorname{sg}\!\left[
\bar p_t(v)\Delta_t(v\mid\vs_t)
\right]
\nabla_\theta\log\pi_\theta(v\mid\vs_t)
-\alpha_{m+1}\nabla_\theta\widehat d_t(\theta)
\Biggr].
\end{equation}
Here, $\operatorname{sg}$ denotes stop-gradient, and $\widehat d_t(\theta)$ is the per-position KL penalty estimator toward the initial student $\pi_{\mathrm{stu}}$, computed using \texttt{verl}'s \texttt{low\_var\_kl} implementation.
Both terms are applied to the same selected response positions. The sampled prefixes, candidate sets, selection mask,
and KL coefficient are held fixed during differentiation.

\section{Proof of the Mass-Invariance Construction}\label{app:mass_proof}

\begin{proposition}[Mass invariance of the Direct-OPD update]
\label{prop:mass}
Fix a state $\vs$ and a student checkpoint $\theta_0$, and write $p(v)=\pi_{\theta_0}(v\mid\vs)>0$ for all $v\in\mathcal V$.
Let $\mathcal V_K(\vs)\subsetneq\mathcal V$ be the top-$K$ candidate set, with $K\geq2$, and define
$\bar p(v)=p(v)/\sum_{u\in\mathcal V_K(\vs)}p(u)$
for $v\in\mathcal V_K(\vs)$.
Let $\mathbf t=(t_v)_{v\in\mathcal V_K(\vs)}$ and
$\mathbf q=(q_v)_{v\in\mathcal V_K(\vs)}$ be distinct, strictly positive probability vectors on the candidate set,
and let $\mathbf b=(b_v)_{v\in\mathcal V\setminus\mathcal V_K(\vs)}$ be a common strictly positive probability vector
on its complement.
For $0<\epsilon<1$, define the teacher and reference distributions:
\begin{equation}
P_\epsilon(v)=
\begin{cases}
\epsilon t_v, & v\in\mathcal V_K(\vs),\\
(1-\epsilon)b_v, & v\notin\mathcal V_K(\vs),
\end{cases}\qquad
Q_\epsilon(v)=
\begin{cases}
\epsilon q_v, & v\in\mathcal V_K(\vs),\\
(1-\epsilon)b_v, & v\notin\mathcal V_K(\vs).
\end{cases}
\label{eq:mass_pair_app}
\end{equation}
Here, $\epsilon$ is the total probability mass that each checkpoint assigns to the student's candidates.
Both checkpoints vary with $\epsilon$, while the student, candidate set, and conditional distributions remain fixed.

Then, for every $\epsilon\in(0,1)$:
\begin{enumerate}[label=(\roman*),leftmargin=*]
\item \emph{Reward invariance.} For every candidate $v\in\mathcal V_K(\vs)$, the Direct-OPD reward (Eq.~\ref{eq:policy_shift}) with $\pi_\mathrm{T}=P_\epsilon$ and $\pi_\mathrm{ref}=Q_\epsilon$ satisfies
\begin{equation}
\Delta_\epsilon(v\mid\vs)
=
\log\frac{P_\epsilon(v)}{Q_\epsilon(v)}
=
\log\frac{t_v}{q_v}.
\label{eq:mass_reward}
\end{equation}
\item \emph{Fixed, nonzero update.} The local reward gradient of Direct-OPD (Eq.~\ref{eq:reward_grad}) at $\theta_0$,
\begin{equation}
\mathbf g_\epsilon
=
\sum_{v\in\mathcal V_K(\vs)}
\bar p(v)\log\frac{t_v}{q_v}
\nabla_\theta\log\pi_\theta(v\mid\vs)
\Big|_{\theta=\theta_0}
=\mathbf g,
\label{eq:mass_gradient}
\end{equation}
is independent of $\epsilon$ and nonzero with respect to the student's logits.
\item \emph{Divergences scale with the mass.} The teacher--reference divergences satisfy
\begin{equation}
D_{\mathrm{JS}}(P_\epsilon,Q_\epsilon)
=\epsilon D_{\mathrm{JS}}(\mathbf t,\mathbf q),
\ 
D_{\mathrm{KL}}(P_\epsilon\Vert Q_\epsilon)
=\epsilon D_{\mathrm{KL}}(\mathbf t\Vert\mathbf q),
\ 
D_{\mathrm{KL}}(Q_\epsilon\Vert P_\epsilon)
=\epsilon D_{\mathrm{KL}}(\mathbf q\Vert\mathbf t).
\label{eq:mass_divergences}
\end{equation}
\end{enumerate}
Consequently, as $\epsilon\to0$, all three divergences vanish, whereas the reward and its gradient remain fixed and nonzero.
\end{proposition}

\begin{proof}
(i) The common factor $\epsilon$ cancels in the ratio.
(ii) By (i), each term of $\mathbf g_\epsilon$ is independent of $\epsilon$, since $\mathcal V_K(\vs)$ and $\bar p$ depend only on $\theta_0$.
For softmax logits $\mathbf z$, $\nabla_{\mathbf z}\log\pi(v\mid\vs)=\mathbf e_v-\pi(\cdot\mid\vs)$, so $\mathbf g=\mathbf w-(\sum_u w_u)\,p$, where $w_v=\bar p(v)\log(t_v/q_v)$ on $\mathcal V_K(\vs)$ and $w_v=0$ elsewhere.
Since $\mathcal V_K(\vs)\subsetneq\mathcal V$ and $p>0$, $\mathbf g=\mathbf 0$ would require $\mathbf w=\mathbf 0$, i.e., $\mathbf t=\mathbf q$.
(iii) Off the candidate set, $P_\epsilon$, $Q_\epsilon$, and their mixture coincide, so the shared tail contributes zero to each divergence; on the candidate set, the factor $\epsilon$ cancels inside each logarithm and factors out of the sum.
\end{proof}


\section{Token Overlap Across Divergence Ranks}
\label{app:token_overlap}

Prior work associates successful OPD with increasing overlap between
the student's and teacher's high-probability token sets
\citep{rethinkopd}.
Direct-OPD instead transfers a teacher--reference log-ratio,
motivating us to examine overlap with both checkpoints.
We track how these overlaps evolve when supervision is restricted
to different teacher--reference JSD deciles.

For a student policy $\pi_S$ and a comparison policy $\pi_C$ at state $s_t$, let
\begin{align}
    \mathcal V_k^S(\vs_t)
    &= \operatorname{TopK}\bigl(\pi_S(\cdot\mid\vs_t), k\bigr), \\
    \mathcal V_k^C(\vs_t)
    &= \operatorname{TopK}\bigl(\pi_C(\cdot\mid\vs_t), k\bigr),
\end{align}
so that the per-state overlap ratio is
\begin{equation}
    \operatorname{Overlap}_k(S,C;\vs_t)
    =
    \frac{
        \left|\mathcal V_k^S(\vs_t)
        \cap
        \mathcal V_k^C(\vs_t)\right|
    }{k}.
\end{equation}
which we report with $k=16$ against both the post-RL teacher ($C = T$) and the reference ($C = T_{\mathrm{ref}}$). Every bin starts from the same student initialization, so the two panels record how training on a given decile moves the student relative to each checkpoint.

\begin{figure}[tbp]
    \centering
    \includegraphics[width=\linewidth]{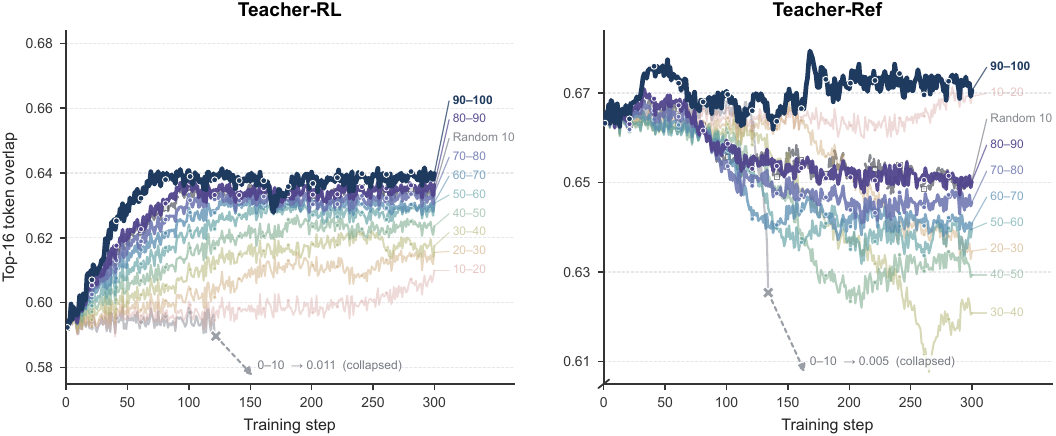}
    \caption{
Top-16 token overlap over training for Qwen3-1.7B under the JustRL teacher pair,
using JSD deciles or random 10\% selection.
Panels compare the student with the post-RL teacher (left) and pre-RL reference (right).
Overlap is the fraction of shared tokens between top-16 sets.
    }
    \label{fig:overlap}
\end{figure}


Higher-JSD bins generally attain greater overlap with the post-RL teacher (Figure~\ref{fig:overlap}, left), broadly tracking the performance ordering in Figure~\ref{fig:ladder_avg_32}.
This resembles the overlap growth observed in successful standard OPD \citep{rethinkopd}, here under an objective that transfers
the teacher's policy shift.
The reference panel adds a complementary observation: the highest-JSD bin increases teacher overlap while maintaining reference overlap above its initial level.
Reference overlap alone does not follow the performance ordering; for example, the 10--20 bin recovers close to its initial reference overlap despite much smaller gains in teacher overlap.

The strongest transfer occurs without a trade-off between overlap with the two checkpoints: the highest-JSD bin increases teacher overlap while sustaining reference overlap.
This suggests that transferring the RL-induced policy shift need not entail moving from the reference's high-probability candidate set.
The teacher--reference log-ratio can favor a token that remains highly ranked under both checkpoints, allowing transfer through changes in relative preference among shared candidates.
Related evidence from standard OPD shows that supervision restricted to tokens shared by the student and teacher recovers nearly the full benefit of student top-$k$ supervision \citep{rethinkopd}.
Our JSD criterion identifies positions with substantial changes in candidate probabilities or total mass relative to the tail (Eq.~\ref{eq:topk_dist}).
Together, these observations suggest a view of selective transfer as learning substantial changes in probability allocation within largely overlapping candidate spaces.

\section{Sensitivity to Retention Ratio}
\label{app:threshold_sensitivity}

Our main experiments retain the 10\% of positions with the highest teacher--reference JSD within each response.
Figure~\ref{fig:percentile} compares retention ratios of 5\%, 10\%, 15\%, and 20\% for Qwen3-1.7B under the JustRL teacher pair.
All four settings yield broadly similar learning curves and substantial gains over the base model on AIME 2024 and AIME 2025.
No retention ratio consistently dominates across both benchmarks and training checkpoints.

\begin{figure}[tbp]
\centering
\includegraphics[width=\linewidth]{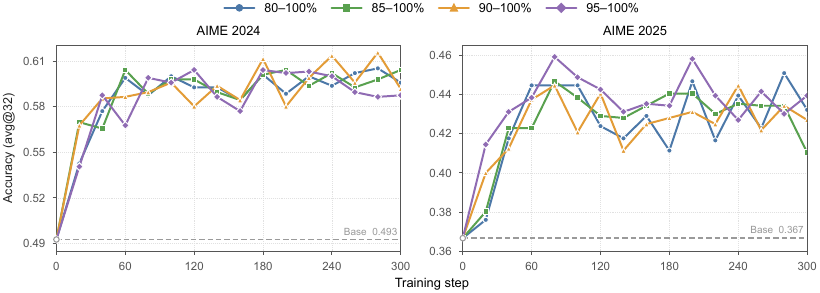}
\caption{
Sensitivity to the per-response retention ratio for Qwen3-1.7B
under the JustRL teacher pair.
The four settings retain the top 20\%, 15\%, 10\%, and 5\% of
positions ranked by teacher--reference JSD.
Dashed lines indicate base-model accuracy.
}
\label{fig:percentile}
\end{figure}


Together with the decile analysis in Section~\ref{sec:section_ladder}, these results reveal an asymmetry in supervision selection.
Training on low-JSD bins can substantially degrade transfer, whereas varying the retained fraction within the high-JSD region has a comparatively small effect.
Moreover, the most selective setting does not consistently outperform the broader subsets.
This pattern supports using JSD to screen supervision, without requiring that larger divergence always imply greater transfer utility.
It also accords with our theoretical motivation: small divergence bounds the observable policy change, while large divergence alone does not establish the value of transferring that change.
The similar performance across retention ratios suggests that effective selection admits a broad operating range in this setting, and we adopt 10\% as a common default.

\definecolor{filteredblue}{HTML}{527F99}
\definecolor{retainedgold}{HTML}{C49336}

\newcommand{\filteredtoken}[1]{%
    \begingroup
    \setlength{\fboxsep}{1pt}%
    \setlength{\fboxrule}{0.5pt}%
    \fcolorbox{filteredblue}{white}{\texttt{#1}}%
    \endgroup
}

\newcommand{\retainedtoken}[1]{%
    \begingroup
    \setlength{\fboxsep}{1pt}%
    \setlength{\fboxrule}{0.5pt}%
    \fcolorbox{retainedgold}{white}{\texttt{#1}}%
    \endgroup
}

\section{Case Study: What Policy Changes Does Selection Preserve?}
\label{app:case_study}

To examine what the mask keeps and removes, we inspect states from 20 rollouts of the initial Qwen3-1.7B student on AIME 2026 under the JustRL pair.
At each state, we compare the teacher and reference next-token distributions over the student's top-$K$ candidates and the residual token, and we mark the state as retained if its JSD falls within the top 10\% of its response, following the per-response rule of Sec.~\ref{sec:methods}.
We write $\Delta p$ for the change in a token's probability from the reference to the teacher.
We selected the four cases below by hand for interpretability, so they illustrate how the score behaves rather than estimate how often each pattern occurs.

\begin{figure}[t]
    \centering
    \includegraphics[width=\linewidth]{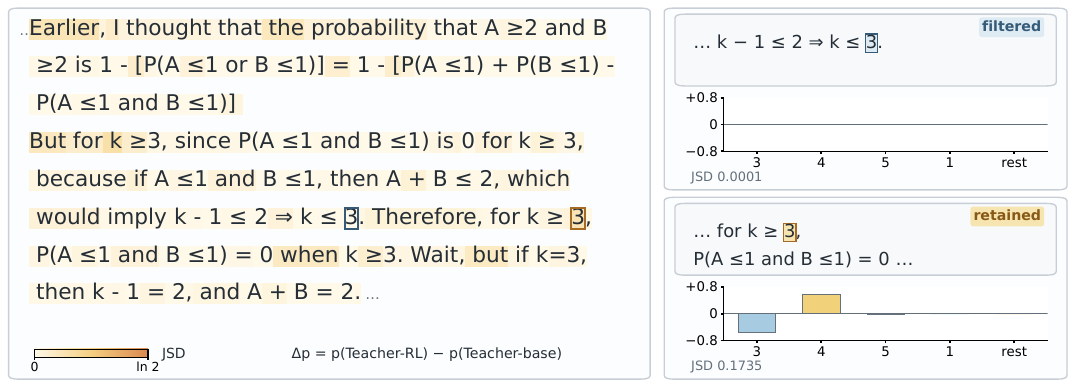}
    \caption{
        \textbf{Same sampled token, different mask decisions.}
        The mask filters a calculation on which the teacher and reference agree but retains supervision for a boundary correction.
    }
    \label{fig:mask-boundary}
\end{figure}

\begin{figure}[t]
    \centering
    \includegraphics[width=\linewidth]{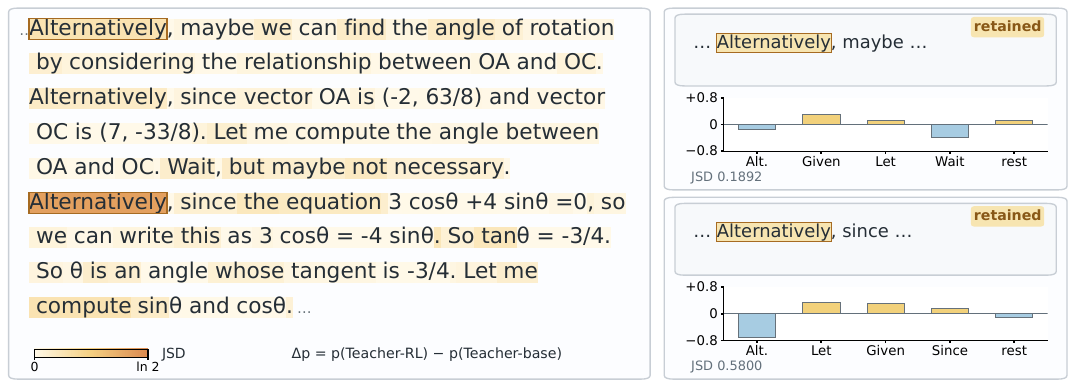}
    \caption{
    \textbf{Retained shifts encode continuation preferences.}
    At two retained states, the teacher shifts probability
    away from \texttt{Alternatively}; the later state favors
    \texttt{Let}, \texttt{Given}, and \texttt{Since}.
    }
    \label{fig:mask-wording}
\end{figure}

\textbf{Selection acts on states, not on sampled tokens.}
Fig.~\ref{fig:mask-boundary} contrasts two states at which the student samples \texttt{3}.
At the first, \filteredtoken{3} completes a valid calculation, and the teacher and reference distributions are nearly identical ($\mathrm{JSD}=0.0001$).
At the second, \retainedtoken{3} sets an incorrect boundary: $k=3$ still permits $A=B=1$, so excluding their simultaneous occurrence requires $k\geq4$.
Here RL moves probability from \texttt{3} to \texttt{4} ($\Delta p\approx-0.55$ and $+0.56$), and the state is retained ($\mathrm{JSD}=0.1735$).
The same token thus receives opposite mask decisions, because the score depends on how the teacher's distribution changed at the state rather than on which token was sampled.

\begin{figure}[t]
    \centering
    \includegraphics[width=\linewidth]{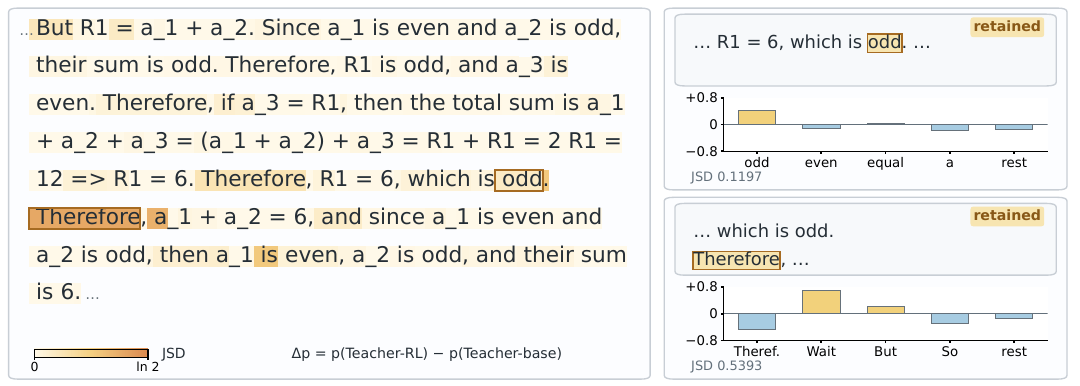}
    \caption{
        \textbf{A later reconsideration cue does not undo
        an earlier adverse shift.}
        The mask retains a shift toward an incorrect parity
        judgment, followed by a teacher preference for
        reconsideration.
    }
    \label{fig:mask-error}
\end{figure}

\textbf{Retained shifts include preferences over how reasoning continues.}
Fig.~\ref{fig:mask-wording} shows two retained states at which the student opens a new passage with \retainedtoken{Alternatively}.
At both, RL lowers the probability of this token; at the later state the decrease is large ($\Delta p\approx-0.71$), and the teacher instead favors \texttt{Let}, \texttt{Given}, and \texttt{Since}.
Together with the arithmetic cases, this shows that the retained supervision covers both local corrections and preferences over how the reasoning proceeds, consistent with a score that measures behavioral change rather than correctness.

\textbf{JSD does not judge correctness.}
In Fig.~\ref{fig:mask-error}, the student describes \texttt{R1 = 6} as \retainedtoken{odd}.
RL reinforces this error, raising the probability of \texttt{odd} ($\Delta p\approx+0.41$) and lowering that of \texttt{even} ($\Delta p\approx-0.12$).
The JSD of 0.1197 exceeds the response's retention cutoff of 0.0744, so this adverse shift is retained.
At the next state, conditioned on the erroneous statement, the teacher favors \texttt{Wait} over \texttt{Therefore} (from 0.01 to 0.71 and from 0.49 to 0.02), a preference compatible with reconsideration, but this later shift does not cancel the reward for \texttt{odd} at the earlier state.
The score thus selects states by the size of the teacher's change, not by whether the change is correct.

\textbf{Summary.}
These cases separate where supervision is applied from what it encourages.
JSD decides which states are retained, while Direct-OPD's log-ratio rewards decide which tokens are encouraged at those states.
Because the score measures how much the teacher's distribution changed on the student's candidates, the retained states include numerical corrections, continuation preferences, and local errors alike.
Masking therefore concentrates Direct-OPD supervision on the states where RL changed the teacher most, without filtering shifts by their correctness.

\section{Training Details}
\label{app:training_details}

\paragraph{Data and prompt.} All runs use the math subset of Skywork-OR1-RL-Data~\citep{skywork_dataset}
and apply the following prompt template.

\begin{quote}
\small\ttfamily
Solve the following math problem step by step.\\
The last line of your response should be of the form\\
Answer: \$Answer (without quotes) where \$Answer is the answer to the problem.\\[0.5em]
\{Question\}\\[0.5em]
Remember to put your answer on its own line after "Answer:".
\end{quote}

\paragraph{Training and evaluation.}
All experiments are implemented with \texttt{verl} and run on 8 NVIDIA H200 GPUs. Table~\ref{tab:training_details} summarizes
the default training and evaluation configuration used throughout
our experiments.
Owing to limited compute, each configuration is trained with a single run; the confidence interval in Sec.~\ref{sec:main_results} resamples held-out problems and does not reflect variation across training seeds.

\begin{table}[t]
    \centering
    \caption{Default training and evaluation configuration.}
    \label{tab:training_details}
    \small
    \renewcommand{\arraystretch}{1.08}
    \setlength{\tabcolsep}{8pt}
    \begin{tabular}{@{}lcc@{}}
        \toprule
        \textbf{Setting} & \textbf{Training} & \textbf{Evaluation} \\
        \midrule
        Framework                 & \texttt{verl}         & -- \\
        Hardware                  & $8\times$ NVIDIA H200 & -- \\
        Global batch size         & 128                   & -- \\
        Mini-batch size           & 128                   & -- \\
        Rollout $n$               & 4                     & -- \\
        Max. prompt length        & 1,024                 & -- \\
        Max. response length      & 2,048                 & 31,744 \\
        Samples per problem       & --                    & 32 \\
        Sampling temperature      & 1.0                   & 0.7 \\
        OPD support size Top-$K$  & 16                    & -- \\
        Top-$p$ sampling          & 1.0                   & 0.95 \\
        Learning rate             & $1\times10^{-6}$      & -- \\
        Training steps            & 300                   & -- \\
        KL coefficient $\alpha$   & Adaptive              & -- \\
        \quad Controller $\epsilon$ & 0.01                 & -- \\
        \quad $[\alpha_{\min}, \alpha_{\max}]$ & $[0.5,\, 2.5]$ & -- \\
        Checkpoint selection      & --                    & AIME 24/25 \\
        Held-out evaluation       & --                    & AIME 26, HMMT Nov.25 / Feb.26  \\
        \bottomrule
    \end{tabular}
\end{table}

\section{Limitations}
Our work has four main limitations.
First, JSD is a proxy for how much RL changed the teacher, not a judge of whether the change is correct.
Because it scores only the magnitude of the teacher--reference difference, S$^2$D-OPD retains a large shift toward an error as readily as a large correction (Fig.~\ref{fig:mask-error}), so the quality of the retained supervision still depends on the teacher's RL.
Selectors that also account for the direction or correctness of a shift, for example through verifier or outcome signals, could filter such states.
Second, JSD is blind to what the student needs.
The score depends on the student only through its top-$K$ candidate set: a state where the student already follows the teacher's shift consumes the same budget as one where it does not, and a low-divergence shift that the student lacks is discarded.
Combining teacher--reference divergence with student-side signals, in the spirit of the student-side calibration in Proxy-OPD~\citep{ProxyOPD}, is a natural extension.
Third, our evidence is limited in scale and scope.
We transfer from 1.5B teachers to students of up to 8B parameters, whereas weak-to-strong transfer is most valuable for much larger students, where RL is most costly.
All experiments use mathematical reasoning with verifiable rewards; whether selective supervision helps in code generation or agentic tasks, where RL-induced shifts may be distributed differently across states, remains untested.
Finally, owing to limited compute, each configuration is trained once, so our confidence interval reflects variation across held-out problems rather than across training seeds; repeated runs would strengthen the per-setting comparisons.

\end{document}